\documentclass[twoside]{article}
\usepackage[accepted]{aistats2019}
\usepackage[round]{natbib}

\ifdefined\usebigfont

\usepackage[fontsize=13pt]{scrextend}
\else
\fi
\usepackage{times}

\usepackage{amsfonts}
\usepackage{graphicx}
\usepackage{amsmath}
\usepackage{amsthm}
\usepackage{amssymb}
\usepackage{xcolor}
\usepackage{thmtools,thm-restate}
\usepackage{booktabs}
\usepackage{caption}
\usepackage{nicefrac}

\usepackage[utf8]{inputenc} 
\usepackage[T1]{fontenc}    
\usepackage{hyperref}       
\usepackage{url}            
\usepackage{booktabs}       
\usepackage{microtype}      

\newcommand{\e}{\exp}
\newcommand{\vect}[1]{\mathbf{#1}}
\newcommand{\what}{\hat{\mathbf{w}}}
\newcommand{\wtilde}{\tilde{\mathbf{w}}}

\newcommand{\wvec}{{\mathbf{w}}}
\newcommand{\rvec}{{\mathbf{r}}}
\newcommand{\bm}[1]{\ensuremath{\boldsymbol{#1}}}
\newcommand{\rhoVec}{\bm{\rho}(t)}

\newcommand{\x}{{\mathbf{x}}}

\newcommand{\E}{\mathbb{E}}

\newcommand{\norm}[1]{\left\lVert#1\right\rVert}
\newcommand{\abs}[1]{\left\vert#1\right\vert}

\newcommand{\xn}{{\mathbf{x}_n}}

\newcommand{\xnT}{{\mathbf{x}_n^\top }}

\newcommand{\sumnsv}{\sum\limits_{n\in \mathcal{S}}}

\newcommand{\set}{\mathcal{S}}

\newcommand{\op}{\vect{P}}

\newtheorem{theorem}{Theorem}

\newtheorem{lemma}{Lemma}

\newtheorem{definition}{Definition}

\newtheorem{corollary}{Corollary}

\newcommand{\dnote}[1]{}
\newcommand{\mnote}[1]{}
\newcommand{\jnote}[1]{}
\newcommand{\remove}[1]{{}}

\newtheorem{assm}{Assumption}
\makeatletter
\newcounter{parentassm}
\newenvironment{subassm}
 {
  \refstepcounter{assm}%
  \protected@edef\theparentassm{\theassm}%
  \setcounter{parentassm}{\value{assm}}%
  \setcounter{assm}{0}%
  \def\theassm{\theparentassm\alph{assm}}%
  \ignorespaces
}{%
  \setcounter{assm}{\value{parentassm}}%
  \ignorespacesafterend
}
\makeatother

\newcommand*{\QEDA}{\hfill\ensuremath{\blacksquare}}%
\newcommand*{\QEDB}{\hfill\ensuremath{\square}}%

\begin{document}
\runningtitle{Stochastic Gradient Descent on Separable Data}
\runningauthor{Mor Shpigel Nacson, Nathan Srebro, Daniel Soudry}

\twocolumn[

\aistatstitle{Stochastic Gradient Descent on Separable Data: \\Exact Convergence with a Fixed Learning Rate}

\aistatsauthor{ Mor Shpigel Nacson\textsuperscript{1} \ \ \  Nathan Srebro\textsuperscript{2} \ \ \ Daniel Soudry\textsuperscript{1}}

\aistatsaddress{ \textsuperscript{1}Technion, Israel,\ \ \ \  \textsuperscript{2}TTI Chicago, USA} 
]

\begin{abstract}
Stochastic Gradient Descent (SGD) is a central tool in machine learning. We prove that SGD converges to zero loss, even with a fixed (non-vanishing) learning rate --- in the special case of homogeneous linear classifiers with smooth monotone loss functions, optimized on linearly separable data. Previous works assumed either a vanishing learning rate, iterate averaging, or loss assumptions that do not hold for monotone loss functions used for classification, such as the logistic loss. We prove our result on a fixed dataset, both for sampling with or without replacement. Furthermore, for logistic loss (and similar exponentially-tailed losses), we prove that with SGD the weight vector converges in direction to the $L_2$ max margin vector as $O(1/\log(t))$ for almost all separable datasets, and the loss converges as $O(1/t)$ --- similarly to gradient descent. Lastly, we examine the case of a fixed learning rate proportional to the minibatch size. We prove that in this case, the asymptotic convergence rate of SGD (with replacement) does not depend on the minibatch size in terms of epochs, if the support vectors span the data. These results may suggest an explanation to similar behaviors observed in deep networks, when trained with SGD.
\end{abstract}

\section{INTRODUCTION}

Deep neural networks (DNNs) are commonly trained using stochastic gradient descent (SGD), or one of its variants. During training, the learning rate is typically decreased according to some schedule (e.g., every $T$ epochs we multiply the learning rate by some $\alpha<1$). Determining the learning rate schedule, and its dependency on other factors, such as the minibatch size, has been the subject of a rapidly increasing number of recent empirical works (\cite{hoffer2017train,Goyal2017,Jastrzebski2017,SmithLe2018} are a few examples). Therefore, it is desirable to improve our understanding of such issues. However, somewhat surprisingly, we observe that we do not have even a satisfying answer to the basic question
\begin{center}
\textit{Why do we need to decrease the learning rate during training?}
\end{center}
At first, it may seem that this question has already been answered. Many previous works have analyzed SGD theoretically (e.g., see \cite{Robbins1951,Bertsekas1999,Geary2001,Bach2011a,Ben-David2014,Ghadimi2013,Bubeck2015,Bottou2016,Ma2017} and references therein), under various assumptions. In all previous works, to the best of our knowledge, one must assume a vanishing learning rate schedule, averaging of the SGD iterates, partial strong convexity (i.e., strong convexity in some subspace), or the Polyak-Lojasiewicz (PL) condition \citep{Bassily2018} --- so that the SGD increments or the loss (in the convex case) will converge to zero for generic datasets. However,  even near its global minima, a neural network loss is not partially strongly convex, and the PL condition does not hold. Therefore, without a vanishing learning rate or iterate averaging, the gradients are only guaranteed to decrease below some constant value, proportional to the learning rate. Thus, in this case, we may fluctuate near a critical point, but never converge to it.

Consequently it may seem that in neural networks we should always decrease the learning rate in SGD or average the weights, to enable the convergence of the weights to a critical point, and to decrease the loss. However, this reasoning does not hold empirically. In many datasets, even with a fixed learning rate and without averaging, we observe that the training loss can converge to zero. For example, we examine the learning dynamics of a ResNet-18 trained on CIFAR10 in Figure \ref{fig: DNN results}. 
Even though the learning rate is fixed, the training loss converges to zero (and so does the classification error).

Notably, we do not observe any convergence issues, as we may have suspected from previous theoretical results.  In fact, if we decrease the learning rate at any point, this only decreases the convergence rate of the training loss to zero. The main benefit of decreasing the learning rate is that it typically improves generalization performance. Such contradiction between existing theoretical and empirical results may indicate a significant gap in our understanding. We are therefore interested in closing this gap.

\begin{figure*}
\begin{centering}
\begin{tabular}{cc}
\includegraphics[width=0.9\columnwidth]{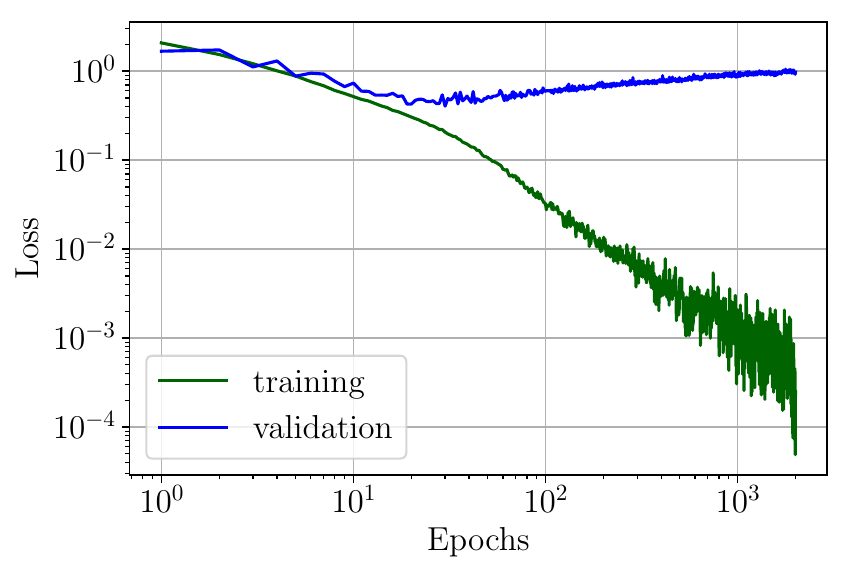}  & \includegraphics[width=0.9\columnwidth]{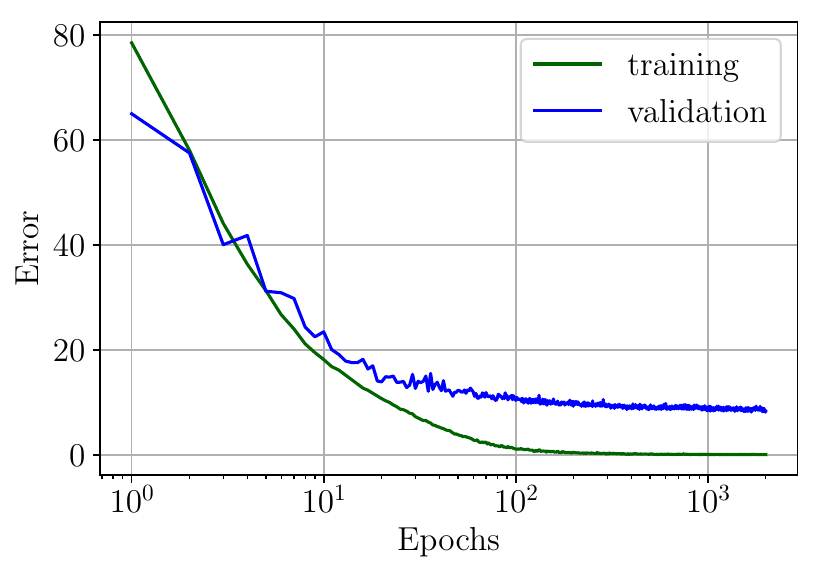}  \tabularnewline
\end{tabular}
\par\end{centering}
\caption{Training of a convolutional neural network on CIFAR10 using stochastic
gradient descent with constant learning rate, softmax
output and a cross entropy loss. We observe that, approximately: (1) The training loss and (classification) error both decays to zero; (2) after a while, the validation loss starts to increase; and (3) in contrast, the validation (classification) error slowly improves. In \citet{soudry2017implicit}, the authors observed similar results with momentum.
\label{fig: DNN results}}
\end{figure*}

\remove{
\begin{table*}
\begin{centering}
\begin{tabular}{|c|c|c|c|c|c|c|}
\hline 
Epoch  & 50  & 100  & 200  & 400  & 2000  & 4000\tabularnewline
\hline 
\hline 
Train loss  & 0.1  & 0.03  & 0.02  & 0.002  & $10^{-4}$  & $3\cdot10^{-5}$\tabularnewline
\hline 
Train error  & 4\%  & 1.2\%  & 0.6\%  & 0.07\%  & 0\%  & 0\%\tabularnewline
\hline 
Validation loss  & 0.52  & 0.55  & 0.77  & 0.77  & 1.01  & 1.18\tabularnewline
\hline 
Validation error  & 12.4\%  & 10.4\%  & 11.1\%  & 9.1\%  & 8.92\%  & 8.9\% \tabularnewline
\hline 
\end{tabular}
\par\end{centering}
\vspace{2mm}
\caption{Sample values from various epochs in the experiment depicted in Fig. \ref{fig: DNN results}. Table modified from \citet{soudry2017implicit}.
 \label{tab:Sample-value-dnn}}
\end{table*}
}
To do so, we first examine the network dynamics in Figure \ref{fig: DNN results}. Since the training error has reached zero after a certain number of iterations, by then the last hidden layer must have become linearly separable. Since the network is trained using the monotone cross-entropy loss (with softmax outputs), by increasing the norm of the weights we decrease the loss. Therefore, if the loss is minimized then the weights would tend to diverge to infinity --- as indeed happens. This weight divergence does not affect the scale-insensitive validation (classification) error, which continues to decrease during training. In contrast, the validation loss starts to increase.

To explain this behavior, \cite{soudry2017implicit,soudry2018journal} focused on the dynamics of the last layer, for a fixed separable input and no bias. For Gradient Descent (GD) dynamics, \cite{soudry2017implicit,soudry2018journal} proved that the training loss converges to zero as $1/t$, the direction of the weight vector converges to the max margin as $1/\log(t)$, and the validation loss increase as $\log(t)$.  This had similar dynamics to those observed in Figure \ref{fig: DNN results}. However, the dynamics of GD are simpler than those of SGD. Notably, it is well known that on smooth functions, for the iterates of GD, the gradient converges to zero even with a fixed learning rate --- just as long as this learning rate is below some fixed threshold (which depends on the smoothness of the function).

\paragraph{Our contributions.} In this paper we examine SGD optimization of homogeneous linear classifiers with smooth monotone loss functions, where the data is sampled either with replacement (the sampling regime typically examined in theory), or without replacement (the sampling regime typically used in practice). For simplicity, we focus on binary classification (e.g., logistic regression). First, we prove three basic results:

\begin{itemize}
\item The norm of the weights diverges to infinity for any learning rate.
\item For a sufficiently small \emph{fixed} learning rate, the loss and gradients converge to zero.
\item This upper bound we derived for the maximal learning rate is proportional to the minibatch size, when the data in SGD is sampled with replacement. 
\end{itemize}
Similar behavior to the last property is also observed in deep networks \citep{Goyal2017,SmithLe2018}. Next, given an additional assumption that the loss function has an exponential tail (e.g., logistic regression), we prove that for almost all linearly separable datasets (i.e., except for measure zero cases):
\begin{itemize}
\item The direction of the weight vector converges to that of the $L_2$ max margin solution.
\item The margin converges as $O(1/\log(t))$, while the training loss converges as $O(1/t)$.
\end{itemize}

These conclusions for SGD are the same as for GD \citep{soudry2017implicit} --- the only difference is the value of the maximal learning rate, which depends on the minibatch size. Therefore, we believe our SGD results might be similarly extended, as GD, to multi-class \citep{soudry2018journal}, other loss functions \citep{Nacson2018}, other optimization methods \citep{gunasekar2018implicit}, linear convolutional neural networks \citep{Gunasekar2018}, and hopefully to nonlinear deep networks. 

Finally, under the assumption that the SVM support vectors span the dataset, we further characterize SGD iterate asymptotic behavior. Specifically, we show that, if we keep the learning rate proportional to the minibatch size, then:
\begin{itemize}
    \item The minibatch size does not affect the asymptotic convergence rate of SGD, in terms of epochs. 
    \item In terms of SGD iterations, the fastest asymptotic convergence rate, is obtained at full batch size, i.e. GD.
\end{itemize}

These results suggest the large potential of parallelism in separable problems, as observed in deep networks \citep{Goyal2017,SmithLe2018}.

\section{PRELIMINARIES \label{sec: prev-results}} 
Consider a dataset $\left\{ \mathbf{x}_{n},y_{n}\right\} _{n=1}^{N}$, with binary labels $y_{n}\in\left\{ -1,1\right\}$ . We analyze learning
by minimizing an empirical loss of homogeneous linear predictors (i.e., without bias), of the form
\begin{equation}
\mathcal{L}\left(\mathbf{w}\right)=\sum_{n=1}^{N}\ell\left(y_{n}\mathbf{w}^{\top}\mathbf{x}_{n}\right)\,,\label{eq: general loss functions}
\end{equation}
where $\mathbf{w}\in\mathbb{R}^{d}$ is the weight vector. To simplify notation, we assume that
$\forall n:\,y_{n}=1$ \textemdash{} this is true without loss of
generality, since we can always re-define $y_{n}\mathbf{x}_{n}$ as
$\mathbf{x}_{n}$.

We are particularly interested in problems that are linearly separable and with a smooth strictly decreasing and non-negative loss function. Therefore, we assume:

{\assm{The dataset is strictly linearly separable: $\exists\mathbf{w}_{*}$
such that $\forall n:\,\mathbf{w}_{*}^{\top}\mathbf{x}_{n}>0$ .\label{assum: Linear sepereability}}}

Given that the data is linearly separable, the maximal $L_2$ margin is strictly positive
\begin{equation}
\gamma=\max_{\mathbf{w}\in \mathbb{R}^d:\left\Vert \mathbf{w}\right\Vert =1}\min_{n}\mathbf{w}^{\top}\mathbf{x}_{n} > 0\,.
\label{eq: max margin gamma}
\end{equation}

{\assm{$\ell\left(u\right)$ is a positive, differentiable, $\beta$-smooth function (\emph{i.e.}, its derivative is $\beta$-Lipshitz), monotonically
decreasing to zero, (so\footnote{The requirement of nonnegativity and that the loss asymptotes to zero
is purely for convenience. It is enough to require the loss is monotone
decreasing and bounded from below. Any such loss asymptotes to some
constant, and is thus equivalent to one that satisfies this assumption,
up to a shift by that constant.} $\forall u:\,\ell\left(u\right)>0,\ell^{\prime}\left(u\right)<0$
and $\lim_{u\rightarrow\infty}\ell\left(u\right)=\lim_{u\rightarrow\infty}\ell^{\prime}\left(u\right)=0$), and $\lim\sup_{u\rightarrow -\infty}\ell^{\prime}\left(u\right) \neq 0$.\label{assum: loss properties}}}

Many common loss functions, including the logistic and probit losses, follow Assumption \ref{assum: loss properties}.
Assumption \ref{assum: loss properties} also straightforwardly implies
that $\mathcal{L}\left(\mathbf{w}\right)$ is a $\beta\sigma_{\max}^{2}$-smooth
function, where the columns of $\mathbf{X}$ are all samples, and
$\sigma_{\max}$ is the
maximal singular value of $\mathbf{X}$.

Under these conditions, the infimum of the optimization problem is
zero, but it is not attained at any finite $\mathbf{w}$. Furthermore,
no finite critical point $\mathbf{w}$ exists. We consider minimizing
eq. \ref{eq: general loss functions} using Stochastic Gradient Descent
(SGD) with a fixed learning rate $\eta$, \emph{i.e., }with steps
of the form: 
\begin{equation}
\mathbf{w}\left(t+1\right)=\mathbf{w}\left(t\right)-\frac{\eta}{B}\sum_{n\in\mathcal{B}\left(t\right)}\ell^{\prime}\left(\mathbf{w}\left(t\right)^{\top}\mathbf{x}_{n}\right)\mathbf{x}_{n},\label{eq: SGD dynamics with B}
\end{equation}
where $\mathcal{B}\left(t\right)\subset\left\{ 1,\dots,N\right\} $
is a minibatch of $B$ distinct indices, chosen so $K=N/B$ is an integer, and that it satisfies one of the following assumptions.
The first option is the assumption of random sampling with replacement:
\begin{subassm}
\begin{assm}\label{assm:SGD Sampling with replacement}
\emph{[Random sampling with replacement]} At each iteration $t$ we randomly and uniformly sample a minibatch $\mathcal{B}(t)$ of $B$ distinct indices, i.e. so each sample has an identical probability to be selected. 
\end{assm}
For example, this assumption holds if at each iteration we uniformly sample the indices without replacement from $\left\{ 1,\dots,N\right\}$, or uniformly sample $k\in\{1,\dots,K \}$ and select $\mathcal{B}\left(t\right)=\mathcal{B}_{k}$, where $\{\mathcal{B}_{k}\}_{k=0}^{K-1}$  is some fixed partition of the data indices, i.e., \[
\cup_{k=0}^{K-1}\mathcal{B}_{k}=\left\{ 1,\dots,N\right\} .
\]
This assumption is rather common in theoretical analysis, but less common in practice.
The next alternative sampling method is more common in practice:
\begin{assm}
[Sampling without replacement] At each epoch, the minibatches partition the data: \[\forall u\in\{0,1,2,\dots\}:\,\cup_{k=0}^{K-1}\mathcal{B}\left(Ku+k\right)=\left\{ 1,\dots,N\right\}.\]\label{assm:SGD Sampling} \end{assm} \end{subassm}  This way, each sample is chosen exactly once at each epoch, and SGD completes balanced passes over the data. An important special case of this assumption is random sampling without replacement, which is the practically common method. Other special cases are periodic sampling (round-robin), and even adversarial selection of the order of the samples.

\section{MAIN RESULT 1: THE LOSS CONVERGES TO A GLOBAL INFIMUM}	

The weight norm always diverges to infinity, for any learning rate, as we prove next. 

\begin{lemma}
Given assumptions \ref{assum: Linear sepereability} and \ref{assum: loss properties}, and any starting point $\mathbf{w}(0)$, the iterates of SGD  on $\mathcal{L}\left(\mathbf{w}\right)$ (eq. \ref{eq: SGD dynamics with B}), with either sampling regimes (Assumption \ref{assm:SGD Sampling with replacement} or \ref{assm:SGD Sampling}),  diverge
to infinity, i.e. $\left\Vert \mathbf{w}\left(t\right)\right\Vert \rightarrow\infty$.
\end{lemma}

\begin{proof}
Since the data is linearly separable, $\exists\mathbf{w}_{*}$ such
that $\forall n:\,\mathbf{w}_{*}\mathbf{x}_{n}>0$. We examine the
dot product of $\mathbf{w}^{*}$ with the iterates of SGD
\[
\mathbf{w}_{*}^{\top}\mathbf{w}\left(t\right)=\mathbf{w}_{*}^{\top}\mathbf{w}\left(0\right)-\frac{\eta}{B}\sum_{u=0}^{t-1}\sum_{n\in\mathcal{B}\left(u\right)}\ell^{\prime}\left(\mathbf{x}_{n}^{\top}\mathbf{w}\left(u\right)\right)\mathbf{w}_{*}^{\top}\mathbf{x}_{n}\,.
\]
Since $\forall n:\,\mathbf{w}_{*}\mathbf{x}_{n}>0$ and $-\ell'(u)> 0$ for any finite $u$, we get that either
$\mathbf{w}_{*}^{\top}\mathbf{w}\left(t\right)\rightarrow\infty$ or
$\ell^{\prime}\left(\mathbf{x}_{n}^{\top}\mathbf{w}\left(u\right)\right)\rightarrow0$.
In the first case, from Cauchy-Shwartz 
\[
\left\Vert \mathbf{w}\left(t\right)\right\Vert \geq\left\Vert \mathbf{w}_{*}^{\top}\mathbf{w}\left(t\right)\right\Vert /\left\Vert \mathbf{w}_{*}\right\Vert \rightarrow\infty\,.
\]
In the second case, since $-\ell^{\prime}\left(u\right)$ is strictly
positive for any finite value, and achieves zero only at $u\rightarrow\infty$,
we must have $\mathbf{x}_{n}^{\top}\mathbf{w}\left(t\right)\rightarrow\infty$,
which again implies 
\[
\left\Vert \mathbf{w}\left(t\right)\right\Vert \geq\left\Vert \mathbf{x}_{n}^{\top}\mathbf{w}\left(t\right)\right\Vert /\left\Vert \mathbf{x}_{n}\right\Vert \rightarrow\infty\,.
\] 
Combing both cases, we prove the theorem.
\end{proof}

As the weights go to infinity, we wish to understand the asymptotic behavior of the loss. As the next theorem shows, if the fixed learning rate $\eta$ is sufficiently small, then we get that the loss converges to zero.

\begin{restatable}{theorem}{LRasymptotic}
\label{thm: Main Theorem 1}Let $\mathbf{w}\left(t\right)$ be the
iterates of SGD (eq. \ref{eq: SGD dynamics with B}) from any starting point $\mathbf{w}(0)$, where samples are either \emph{(case 1)} selected randomly with replacement (Assumption \ref{assm:SGD Sampling with replacement})) and with learning rate \begin{equation}
\frac{\eta}{B}<\frac{2\gamma^{2}}{\beta\sigma_{\max}^{4}}\,, \label{eq: eta condition with replacement}
\end{equation}
or \emph{(case 2)} sampled without replacement (Assumption \ref{assm:SGD Sampling})) and with  learning rate \begin{equation}
\frac{\eta}{B}<\min\left[\frac{1}{2K\beta\sigma_{\max}^{2}},\frac{\gamma}{2\beta\sigma_{\max}^{3}\left( K+\gamma^{-1}\sigma_{\max}\right)}\right] \,. \label{eq: eta condition}
\end{equation}
For linearly separable data (Assumption \ref{assum: Linear sepereability}), and smooth-monotone loss function (Assumption \ref{assum: loss properties}), we have the following, almost surely (with probability $1$) in the first case, and surely in the second case:
\begin{enumerate}
\item The loss converges to zero:  \[\lim_{t\rightarrow\infty}\mathcal{L}\left(\mathbf{w}\left(t\right)\right)=0,\]
\item All samples are correctly classified, given sufficiently long time: \[\forall n: \lim_{t\rightarrow\infty}\mathbf{w}\left(t\right)^{\top}\mathbf{x}_{n}=\infty,\]
\item The iterates of SGD are square summable: \[\sum_{t=0}^{\infty}\left\Vert \mathbf{w}\left(t+1\right)-\mathbf{w}\left(t\right)\right\Vert ^{2}<\infty.\] 
\end{enumerate}
\end{restatable}

The complete proof of this theorem is given in section \ref{sec:proof of theorem 1} in the appendix. The proof relies on the following key lemma
\begin{lemma}\label{lem: min NN PCA}
The $L_{2}$ max margin lower bounds the minimal ``non-negative
right eigenvalue'' of $\mathbf{X}$
\begin{equation} 
\gamma=\max_{\mathbf{w}\in\mathbb{R}^{d}:\left\Vert \mathbf{w}\right\Vert =1}\min_{n}\mathbf{w}^{\top}\mathbf{x}_{n} \leq \min_{\mathbf{v}\in\mathbb{R}_{\geq0}^{d}:\left\Vert \mathbf{v}\right\Vert =1}\left\Vert \mathbf{X}\mathbf{v}\right\Vert \label{eq: min NN PCA}
\end{equation}
\end{lemma}
\begin{proof} In this proof we define \textbf{$\mathbf{v^*}$ }as the minimizer
of the right hand side of eq. \ref{eq: min NN PCA}, and $\mathbf{w}_{*}$ as the maximizer of the optimization problem on the left hand side of the same equation. On the one hand 
\begin{equation}
\mathbf{w}_{*}^{\top}\mathbf{X}\mathbf{v}^*\overset{\left(1\right)}{\leq}\left\Vert \mathbf{w}_{*}\right\Vert \left\Vert \mathbf{X}\mathbf{v}^*\right\Vert \overset{\left(2\right)}{=}\min_{\mathbf{v}\in\mathbb{R}_{\geq0}^{d}:\left\Vert \mathbf{v}\right\Vert =1}\left\Vert \mathbf{X}\mathbf{v}\right\Vert \,,\label{eq: wXv 1}
\end{equation}
where in $\left(1\right)$ we used Cauchy-Shwartz inequality, and
in $\left(2\right)$ we used the definition of $\mathbf{v}^*$, and
that $\left\Vert \mathbf{w}_{*}\right\Vert =1$. On the other hand,
\begin{equation}
\mathbf{w}_{*}^{\top}\mathbf{X}\mathbf{v}^*\overset{\left(1\right)}{\geq}\gamma\sum_{n=1}^{N}v_{n}^*\overset{\left(2\right)}{\geq}\gamma\sqrt{\sum_{n=1}^{N}(v_{n}^*)^2} \overset{\left(3\right)}{=}\gamma\,,\label{eq: wXv 2}
\end{equation}
where in $\left(1\right)$ we used the definition of the $L_{2}$ max
margin from the left hand side of eq. \ref{eq: min NN PCA} and $\mathbf{v}^* \in \mathbb{R}^d_+$, in $\left(2\right)$ we used
that $v_{n}\geq0$ and the triangle inequality, and in  $\left(3\right)$  we used that $\left\Vert \mathbf{v}\right\Vert =1$. Together,
eqs. \ref{eq: wXv 1} and \ref{eq: wXv 2} imply the Lemma.
\end{proof}

This Lemma is useful since the SGD weight increments in eq. \ref{eq: SGD dynamics with B} have the form $\mathbf{Xv}$, where $\mathbf{v}$ is some vector with non-negative components. This enables us to bound the norm of the SGD updates using the norm of the full gradient, which allows us to use similar analysis as for GD. Additionally, we note the regime we analyze in Theorem \ref{thm: Main Theorem 1} is somewhat unusual,  as the weight vector goes to infinity. In many previous works it is assumed that there exists a finite critical point, or that the weights are bounded within a compact domain.

\paragraph{Theorem \ref{thm: Main Theorem 1} Implications.} In both sampling regimes, we obtained that a fixed (non-vanishing) learning rate results in convergence to zero error. In the case of random sampling with replacement (Assumption \ref{assm:SGD Sampling with replacement}) we got a better upper bound on the learning rate (eq. \ref{eq: eta condition with replacement}), which does not depend on $K$. Interestingly, this bound matches the empirical findings of  \citet{Goyal2017,SmithLe2018}, which observed that in a large range $\eta \propto B$. Interestingly, in our case the relation $\eta \propto B$ holds exactly for all $B$ in the maximum learning rate (eq. \ref{eq: eta condition with replacement}). In contrast, for linear regression, the relation becomes sub-linear for large $B$  \citep{Ma2017}.

We also considered here the case when the datapoints are sampled without replacement (Assumption \ref{assm:SGD Sampling}). This is in contrast to most theoretical SGD results, which typically assume sampling with replacement (which is less common in practice). There are a few notable exceptions (\cite{Geary2001,Bertsekas2015,Shamir2016}, and references therein). Perhaps the most similar previous result is the classical result of (Proposition 2.1 in \cite{Geary2001}), which has a similar sampling schedule, and in which the weights can go to infinity. However, in this result the learning rate must go to zero for the SGD iterates to converge. In our case, we are able to relax this assumption since we focus on linear classification with a monotone loss and separable data. 

When assuming sampling without replacement (Assumption \ref{assm:SGD Sampling}) the learning rate bound  (eq. \ref{eq: eta condition}) becomes significantly lower --- roughly proportional to $1/K$. This is because such a sampling assumption is very pessimistic (e.g., the samples can be selected by an adversary). Therefore, a small (yet non vanishing) learning rate is required to guarantee convergence. Such a dependence on $K$ is expected, since in this case we need to use a incremental gradient method type of proof, where such low learning rates are common. For example, in \citet{Bertsekas2015} Proposition 3.2b, to get a low final error we must have a learning rate $\eta \ll 1/K^2$.



\section{MAIN RESULT 2: THE WEIGHT VECTOR DIRECTION CONVERGES TO THE MAX MARGIN}	\begin{figure*}
\begin{centering}
\includegraphics[width=1\textwidth]{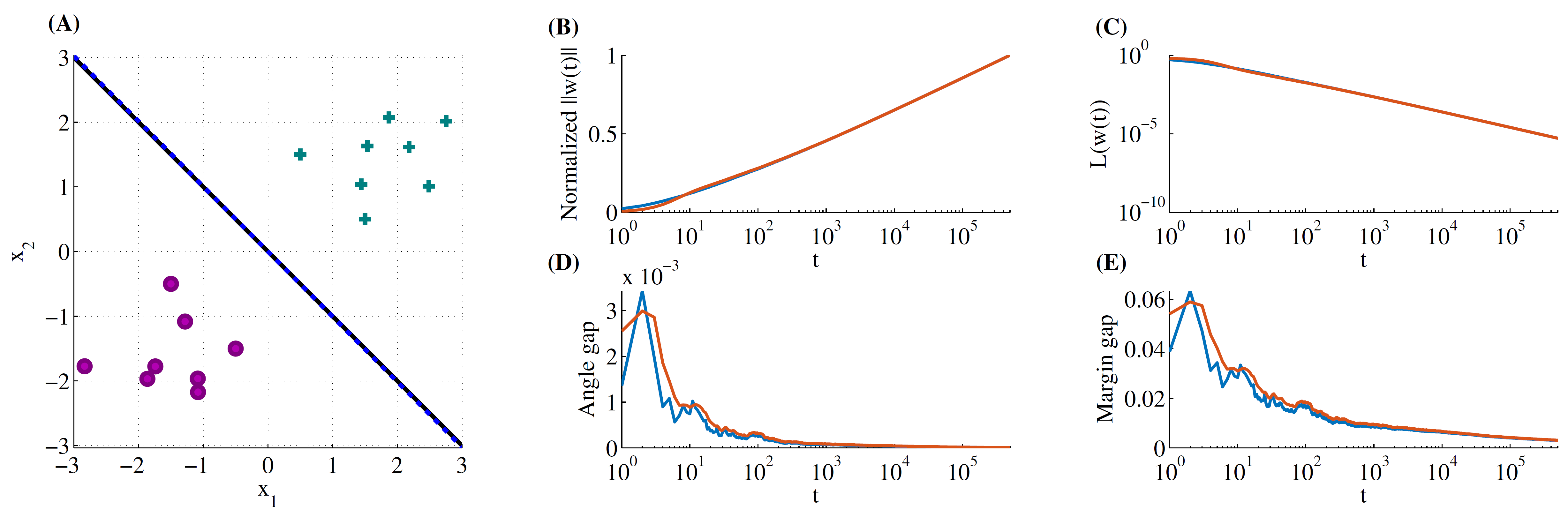} 
\par\end{centering}
\caption{\textbf{Visualization of Theorem \ref{thm: direction convergence to SVM} on a synthetic dataset in which the
$L_{2}$ max margin vector $\hat{\mathbf{w}}$ is precisely known}.
\textbf{(A)} The dataset (positive and negatives samples ($y=\pm1$)
are respectively denoted by $'+'$ and $'\circ'$), max margin separating
hyperplane (black line), and the asymptotic solution of SGD (dashed
blue). For both SGD (blue line) and SGD with momentum (orange line), we show: \textbf{(B)
}The norm of $\mathbf{w}\left(t\right)$, normalized so it would equal
to $1$ at the last iteration, to facilitate comparison. As expected
(from eq. \ref{define wVec}), the norm increases logarithmically;
\textbf{(C) }the training loss. As expected, it decreases as $t^{-1}$
(eq. \ref{eq: logistic loss convergence}); and \textbf{(D\&E) }the
angle and margin gap of $\mathbf{w}\left(t\right)$ from $\hat{\mathbf{w}}$
(eqs. \ref{eq: angle} and \ref{eq: margin}). As expected, these
are logarithmically decreasing to zero. Figure reproduced from \citet{soudry2017implicit}. We also observe similar behavior with different input dimension $d$. This is demonstrated in Figure \ref{fig: SGD results in higher dimension} \label{fig:Synthetic-dataset}}
\end{figure*}
Next, we focus on a special case of monotone loss functions:
	\begin{definition} \label{def: exponential tail} A function $f(u)$ has a ``tight exponential tail", if there exist positive constants $\mu_+,\mu_-$, and $\bar{u}$ such that $\forall u>\bar{u}$:
		$$(1-\e(-\mu_- u))e^{-u} \le f(u) \le (1+\e(-\mu_+ u))e^{-u} $$ 
	\end{definition}
	\begin{assm} \label{assum: tight exp-tail}
	  The negative loss derivative $-\ell'(u)$ has a tight exponential tail.
	\end{assm}
    
Specifically, this applies to the logistic loss function. Given this additional assumption, we prove that SGD converges to the $L_2$ max margin solution.  
    \begin{theorem} \label{thm: direction convergence to SVM}
	For almost all datasets for which the assumptions of Theorem \ref{thm: Main Theorem 1} hold, if $-\ell'(u)$ has a tight exponential tail (Assumption \ref{assum: tight exp-tail}), then the iterates of SGD, for any $\wvec(0)$, will behave as:
	\begin{equation} \label{define wVec}
	\wvec(t) = \what \log\left(\frac{\eta}{B}\cdot \frac{t}{K}\right)+\bm{\rho}(t),
	\end{equation}
	where $\hat{\mathbf{w}}$ is the following $L_{2}$ max margin separator:
	\begin{equation}
	\hat{\mathbf{w}}=\underset{\mathbf{\mathbf{w}}\in\mathbb{R}^{d}}{\mathrm{argmin}}\left\lVert \mathbf{w}\right\rVert^2 \,\,\mathrm{s.t.}\,\,\mathbf{w}^{\top}\mathbf{x}_{n}\geq1,
	\end{equation}
	and the residual $\Vert\bm{\rho}(t)\Vert$  is bounded almost surely in the first case of Theorem \ref{thm: Main Theorem 1} (random sampling with replacement), or surely in the second case (sampling without replacement). 
\end{theorem}

Thus, from Theorem \ref{thm: direction convergence to SVM}, for almost any linearly separable data set (e.g., with probability 1 if the data is sampled from an absolutely continuous distribution) , the normalized weight vector converges to the normalized max margin vector, i.e., 
    \[\lim_{t\rightarrow\infty}\frac{\mathbf{w}\left(t\right)}{\left\Vert \mathbf{w}\left(t\right)\right\Vert }=\frac{\hat{\mathbf{w}}}{\left\Vert \hat{\mathbf{w}}\right\Vert }\]
with rate $1/\log(t)$, identically to GD \citep{soudry2017implicit}. Interestingly, the number of minibatches per epoch $K$ affects only the constants. Intuitively, this is reasonable, since if we rescale the time units, then the log term in eq. \ref{define wVec} will only add a constant to the residual $\rho(t)$.

\paragraph{Proof idea.} The theorem is proved in appendix section \ref{sec: proof of theorem 2}.
The proof builds on the results of  \cite{soudry2017implicit} for GD: as the weights diverge, the loss converges to zero, and only the gradients of the support vector remain significant. This implies that the gradient direction, as a positive linear combination of support vectors converges to the direction of the max margin. The main difficulty in extending the proof to the case of SGD is that at each iteration, $\wvec(t)$ is updated using only a subset of the data points. This could potentially lead to large difference from the GD solution. However, conceptually, we show that this difference of $\mathbf{w}(t)$ from the GD dynamics solution is $O(1)$ in $t$. The main novel idea here is that in order to calculate this $O(1)$ difference at time $t$, we use information on sampling selections made in the future, i.e. at times larger than $t$.

\paragraph{Convergence Rates.} Theorem \ref{thm: direction convergence to SVM} directly implies the same convergence rates as in GD \citep{soudry2017implicit}. Specifically, in the $L_2$ distance
\begin{equation}
\left\Vert \frac{\mathbf{w}\left(t\right)}{\left\Vert \mathbf{w}\left(t\right)\right\Vert }-\frac{\hat{\mathbf{w}}}{\left\Vert \hat{\mathbf{w}}\right\Vert }\right\Vert =O\left(\frac{1}{\log t}\right)\,,\label{eq: normalized weight vector}
\end{equation}
in the angle
\begin{equation}
1-\frac{\mathbf{w}\left(t\right)^{\top}\hat{\mathbf{w}}}{\left\Vert \mathbf{w}\left(t\right)\right\Vert \left\Vert \hat{\mathbf{w}}\right\Vert }=O\left(\frac{1}{\log^2 t}\right)\,,\label{eq: angle}
\end{equation}
and in the margin gap
\begin{equation}
\frac{1}{\left\Vert \hat{\mathbf{w}}\right\Vert }-\frac{\min_{n}\mathbf{x}_{n}^{\top}\mathbf{w}\left(t\right)}{\left\Vert \mathbf{w}\left(t\right)\right\Vert }=O\left(\frac{1}{\log t}\right)\,.\label{eq: margin}
\end{equation}
On the other hand, the loss itself decreases as
\begin{equation}
\mathcal{L}\left(\mathbf{w}\left(t\right)\right)=O\left({\frac{1}{t}}\right)\,.\label{eq: logistic loss convergence}
\end{equation}
In Figure \ref{fig:Synthetic-dataset} we visualize these results. Additionally, in Figure \ref{fig: minibatch} we observe that the convergence rates remain nearly the same for different minibatch sizes --- as long as we linearly scale the learning rate with the minibatch size, i.e. $\eta \propto B$. This behavior fits with the behavior of the maximal learning rate for which SGD converge in the case of sampling with replacement (eq. \ref{eq: eta condition with replacement}). However, it is not clear from Theorem \ref{thm: direction convergence to SVM} why the convergence rate stays almost exactly the same with such a linear scaling, since we do not know how does $\rho(t)$ depends on $\eta$ and $B$.
In the special case where the SVM support vectors span the dataset, we can further characterize $\rho(t)$ asymptotic dependence on $\eta$ and $B$. We define $\op\in\mathbb{R}^{d\times d}$ as the orthogonal projection matrix to the subspace spanned by the support vectors, and $\bar{\op}=\vect{I}-\op$ as the complementary projection. In addition, we denote $\alpha_{n}$
as the SVM dual variables so $\what=\sum_{n\in\set}\alpha_{n}\xn$.

\begin{restatable}{theorem}{LRasymptotic2}
\label{thm: refined Theorem} Under the conditions and notation of
Theorem \ref{thm: direction convergence to SVM}, for almost all datasets, if in addition the support vectors
span the data (\emph{i.e.}~$\mathrm{rank}\left(\mathbf{X}_{\set}\right)=\mathrm{rank}\left(\mathbf{X}\right)$, 
where $\mathbf{X}_{\set}$ is a matrix whose columns are only those data points $\x_n$ s.t.~$\what^\top\x_n=1$), then $\lim_{t\rightarrow\infty}\boldsymbol{\rho}\left(t\right)=\tilde{\mathbf{w}}$,
where $\tilde{\mathbf{w}}$ is a solution to 
\begin{equation}
\forall n\in\set:\,\exp\left(-\mathbf{x}_{n}^{\top}\tilde{\mathbf{w}}\right)=\alpha_{n}\,, \ \bar{\op}\left(\wtilde-\wvec(0)\right)=0\,.\label{eq: w tilde}
\end{equation}
\end{restatable}
The theorem is proved in appendix section \ref{sec: proof of theorem 3}. Note that $\wtilde$ is only dependent on the dataset and the initialization. This fact enables us to state the following result for the asymptotic behavior of SGD.
\begin{corollary} \label{corollary: GD refined result}
    Under the conditions and notation of Theorem \ref{thm: refined Theorem}, SGD iterate will behave as:
    \[
	    \wvec(t) = \what \log\left(\frac{\eta}{B}\cdot \frac{t}{K}\right)+\wtilde+o(1)\,,
	\]
	where $\what$ is the maximum-margin separator, $\wtilde$ is the solution of eq. \ref{eq: w tilde} (which does not depend on $K,\, \eta$ and $B$), and $o(1)$ is a vanishing term. Therefore, if the step size is kept proportional to the minibatch size, \textit{i.e.}, $\eta\propto B$, changing the number of minibatches $K$ is equivalent to linearly re-scaling the time units of $t$.
\end{corollary}

From the corollary, we expect the same asymptotic convergence rates for all batch sizes $B$ as long as we scale the learning rate linearly with the batch size, \textit{i.e.}, keep $\eta\propto B$. This is exactly the behavior we observe in Figure \ref{fig: minibatch}. Since changing the number of minibatches is equivalent to linearly re-scaling the time units, smaller $K$ implies faster asymptotic convergence assuming full parallelization capabilities (i.e. the minibatch size does not affect the iterate time). Additionally, note that the corollary only guarantees the same asymptotic behavior. Particularly, different initializations and datasets can exhibit different behavior initially. It remains an interesting direction for future work to understand $\rho(t)$ dependence on $\eta$ and $B$, in the case when the support vectors do not span the dataset.

Lastly, for logistic regression loss, the validation loss (calculated on an independent validation set $\mathcal{V}$) increases as
\[\mathcal{L}_{\mathrm{val}}\left(\mathbf{w}\left(t\right)\right)=\sum_{\x\in\mathcal{V}}\ell\left(\mathbf{w}\left(t\right)^{\top}\mathbf{x}\right)=\Omega(\log (t) ).\]
Notably, as was observed in \cite{soudry2017implicit}, these asymptotic rates also match what we observe numerically for the convnet in Figure \ref{fig: DNN results}: the training loss decreases as $1/t$, the validation loss increases as $\log(t)$, and the validation (classification) improves very slowly, similarly to the logarithmic decay of the angle gap (so the convnet might have a similarly slow decay to its respective implicit bias). 
\remove{
\begin{figure*}
\begin{centering}
\includegraphics[width=1\textwidth]{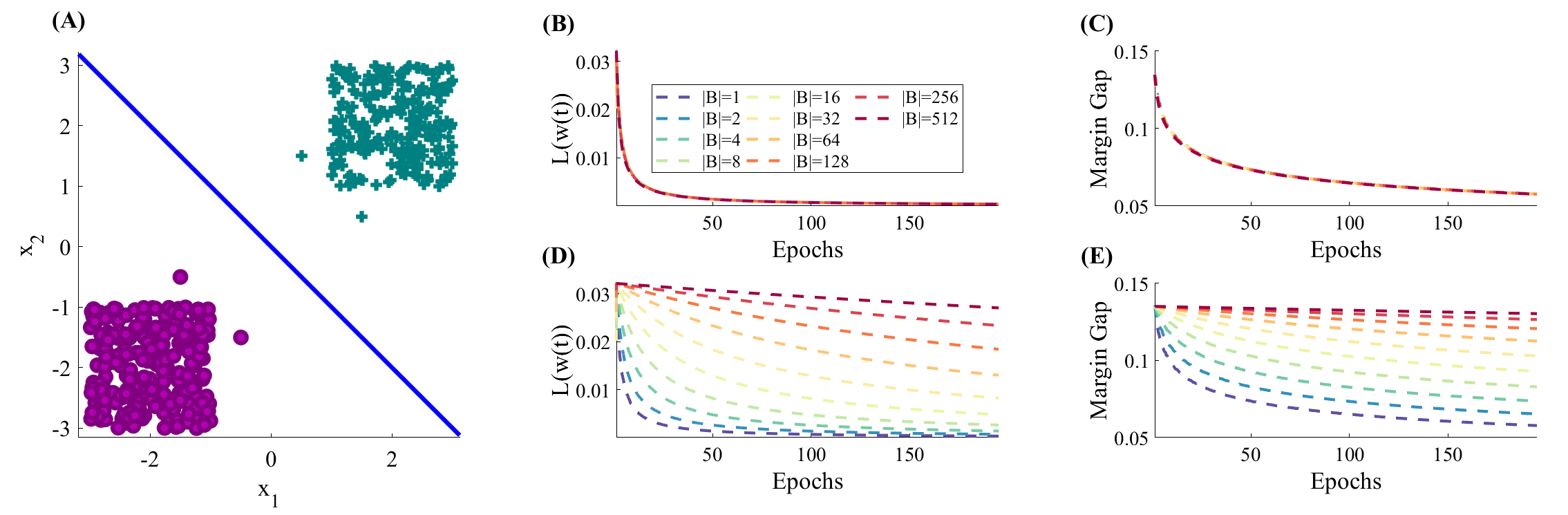} 
\par\end{centering}
\caption{\textbf{We observe the convergence rate of SGD remains almost exactly the same for all minibatch sizes when the learning rate is proportional to the minibatch size} ($\eta=0.01NB$ in figures B and C Vs $\eta=0.01N$ in figures D and E). We initialized $\mathbf{w}(0)$ to be a standard normal vector, and then (for all $B$) we did a single GD iteration with a small learning rate of $\eta=0.01$. The reason for this single iteration is that we wanted to avoid the initial instability of the high learning rate. We used a dataset  \textbf{(A)} with $N=512$ samples divided into two classes, and with the same support vectors as in Figure \ref{fig:Synthetic-dataset}. The convergence of the loss \textbf{(B)} and margin \textbf{(C)} is practically identical for all minibatch sizes. When we used a fixed learning rate, the convergence rate was different \textbf{(D-E)}. } \label{fig: minibatch}
\end{figure*}
}
\begin{figure*}
\begin{centering}
\includegraphics[width=1\textwidth]{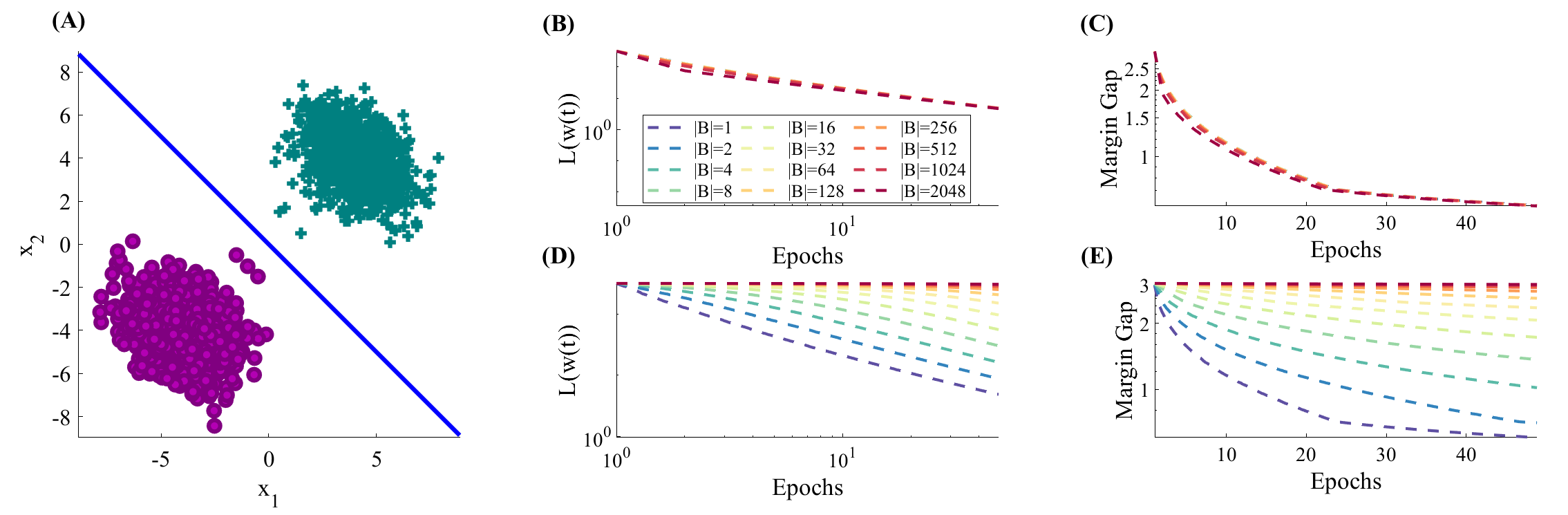} 
\par\end{centering}
\caption{\textbf{We observe the convergence rate of SGD remains almost exactly the same for all minibatch sizes when the learning rate is proportional to the minibatch size} ($\eta=\frac{2\gamma^{2}}{\beta\sigma_{\max}^{2}}B$ in panels \textbf{B} and \textbf{C}, vs. $\eta=\frac{2\gamma^{2}}{\beta\sigma_{\max}^{2}}$ in panels \textbf{D} and \textbf{E}). We initialized $\mathbf{w}(0)$ to be a standard normal vector. We used a dataset  \textbf{(A)} with $N=2048$ samples divided into two classes, and with the same support vectors as in Figure \ref{fig:Synthetic-dataset}. The convergence of the loss \textbf{(B)} and margin \textbf{(C)} is practically identical for all minibatch sizes. When we used a fixed learning rate, the convergence rate was different \textbf{(D-E)}. } \label{fig: minibatch}
\end{figure*}

\section{DISCUSSION AND RELATED WORKS}
In Theorem \ref{thm: Main Theorem 1} we proved that for monotone smooth loss functions on linearly separable data, the iterates of SGD with a sufficiently small (but non-vanishing) learning rate converge to zero loss. In contrast to typical convergence to finite critical points, in this case, the "noise" inherent in SGD vanishes asymptotically. Therefore, we do not need to decrease the learning rate, or average the SGD iterates, to ensure exact convergence. Decaying the learning rate during training will only decrease the convergence speed of the loss. 

To the best of our knowledge, such exact convergence result previously required that either (1) the loss function is partially strongly convex,  i.e. strongly convex except on some subspace (where the dynamics are frozen), as shown in \citep{Ma2017} for the case of over-parameterized linear regression (with more parameters then samples); or (2) that the Polyak-Lojasiewicz (PL) condition applies \citep{Bassily2018}. However, in this paper we do not require such conditions, which does not hold for deep networks, even in the vicinity of the (finite or infinite) critical points. Moreover, the dependence of the learning rate on the minibatch size is different, as we discuss next. 

We proved Theorem \ref{thm: Main Theorem 1} both for random sampling with replacement (Assumption \ref{assm:SGD Sampling with replacement}) and for sampling without replacement (Assumption \ref{assm:SGD Sampling}). In the first case, eq. \ref{eq: eta condition with replacement} implies that, to guarantee convergence, we need to increase the learning rate proportionally to the minibatch size. In the second case (sampling without replacement) the learning rate bound (eq. \ref{eq: eta condition}) is more pessimistic, since our assumption is more general (e.g., it includes adversarial sampling).

In Theorem \ref{thm: direction convergence to SVM}, we proved, given the additional assumption of an exponential tail (e.g., as in logistic regression), that for almost all datasets the weight vector converges to the $L_2$ max margin in direction as $1/\log(t)$, and that the training loss converges to zero as $1/t$. We believe these results could be extended for every dataset, using the techniques of \cite{soudry2018journal}. Again, decaying the learning rate will only degrade the convergence speed to the max margin direction. In fact, the results of \cite{Nacson2018} indicate that we may need to \emph{increase} the learning rate to improve convergence: For GD, \cite{Nacson2018} proved that this can drastically improve the convergence rate from $1/\log(t)$ to $\log(t)/\sqrt{t}$. It is yet to be seen if such results might also be applied to deep networks.

In Theorem \ref{thm: refined Theorem} we further characterized the weights asymptotic behaviour under the additional assumption that the SVM support vectors span the dataset. Combining the results from Theorem \ref{thm: direction convergence to SVM} and Theorem \ref{thm: refined Theorem} we obtain Corollary \ref{corollary: GD refined result}. This corollary states that, under linear scaling of the learning rate with the batch size, the asymptotic convergence rate of SGD, in terms of epochs, is not affected by the mini-batch size.

Thus, we have shown that exact linear scaling of the learning rate with the minibatch size ($\eta\propto B$) is beneficial in two ways: (a) in Theorem \ref{thm: Main Theorem 1} for the upper bound of the learning rate in the case of of random sampling with replacement (b) in Corollary \ref{corollary: GD refined result} for the asymptotic behaviour of the weights assuming tight exponential loss function and that the SVM support vectors span the data.
This exact linear scaling, stands in contrast to previous theoretical results with exact convergence \citep{Ma2017}, in which there exists a "saturation limit". Above this limit we should not increase the learning rate linearly with the minibatch size, or the convergence rate will be degraded, and eventually we will loose the convergence guarantee. 
As predicted by Corollary \ref{corollary: GD refined result}, in Figure \ref{fig: minibatch} we observe that with a linear scaling $\eta\propto B$, the convergence plots exactly match: as we can see, there is almost no asymptotic difference between different minibatch sizes. Therefore, in contrast to \citet{Ma2017}, there is no "optimal" minibatch size. In this case, to minimize the number of SGD iterations we should use the largest minibatch possible. This will speed up convergence in wall clock time (as was done in \citet{Goyal2017,SmithLe2018}) if it is possible to parallelize the calculation of a minibatch  --- so one SGD update with a minibatch of size $MB$ takes less time then $M$ updates of SGD with minibatch of size $B$.

An early version of this manuscript previously appeared on arxiv. However, it had only the results in the case of sampling without replacement, and no Theorem \ref{thm: refined Theorem}. Two other related SGD results appeared on arXiv in parallel (with less than a week difference). 

First, \citet{ji2018risk} analyzed logistic regression optimized by SGD on separable data (in addition to other results on GD when the data is non-separable). \citet{ji2018risk} also assume a fixed learning rate, but use averaging of the iterates (which is known to enable exact convergence). They focus on the case in which the datapoints are independently sampled from a separable distribution, while we focused on the case of sampling from a fixed dataset. They show, that with high probability, the population risk converges to zero as $\tilde{O}(1/t)$. As explained in \citet{ji2018risk}, such a fast rate was proven before only for strongly convex loss functions (the logistic loss is not strongly convex). We showed a similar rate, but for the empirical risk (eq. \ref{eq: logistic loss convergence}). We additionally showed that the weight vector converges in direction to the direction of the $L_2$ max margin. 

Second, among other results, \citet{Xu2018} also examined optimizing logistic regression with SGD on a fixed dataset using random sampling with replacement, iterate averaging and a vanishing learning rate. There, in Theorems 3.2 and 3.3, it is shown that the expectation of the loss converges as $\tilde{O}(1/t)$ and the expectation of the averaged iterates converges in the norm as $O(1/\sqrt{\log(t)}$, which is slower than our result. Thus, in contrast to both works \cite{ji2018risk,Xu2018}, we did not assume iterate averaging or decreasing learning rate. Additionally, our new results on sampling with replacement give a linear relationship between the learning rate and the minibatch size, and Corollary \ref{corollary: GD refined result} shows the affect of the minibatch size on the asymptotic convergence rate.

\section{CONCLUSIONS}
We found that for logistic regression with no bias on separable data, SGD behaves similarly to GD in terms of the implicit bias and convergence rate. The only difference is the maximum possible learning rate should change proportionally to the minibatch size. It remains to be seen if this also holds for deep networks.


\subsubsection*{Acknowledgements}

The authors are grateful to C. Zeno, and I. Golan for
helpful comments on the manuscript. This research was supported by the Israel Science foundation (grant No. 31/1031), and by the Taub foundation. A Titan Xp used for this research was donated by the NVIDIA Corporation. NS was partially supported by NSF awards IIS-1302662 and IIS-1764032.

\bibliographystyle{plainnat_no_URL}
	\newpage
	\onecolumn
	\appendix

\part*{Appendix}

For simplicity of notation, in the appendix we absorb the constant $1/B$ in the SGD dynamics into the learning rate dynamics:
\begin{equation}
\mathbf{w}\left(t+1\right)=\mathbf{w}\left(t\right)-\eta\sum_{n\in\mathcal{B}\left(t\right)}\ell^{\prime}\left(\mathbf{w}\left(t\right)^{\top}\mathbf{x}_{n}\right)\mathbf{x}_{n},\label{eq: gradient descent linear}
\end{equation}
In the main paper, we modify the results proven in the appendix to fit the SGD dynamics in eq. \ref{eq: SGD dynamics with B}. 

\section{Proof of Theorem \ref{thm: Main Theorem 1} \label{sec:proof of theorem 1}}	

Our proof relies on Lemma \ref{lem: min NN PCA}. Specifically, since we assumed $\ell^{\prime}(u)<0$, this Lemma implies that
\begin{equation}  \left\Vert \nabla\mathcal{L}\left(\mathbf{w}\left(t\right)\right)\right\Vert =\left\Vert \sum_{n=1}^{N}\ell^{\prime}\left(\mathbf{x}_{n}^{\top}\mathbf{w}\left(t\right)\right)\mathbf{x}_{n}\right\Vert \geq\gamma\sqrt{\sum_{n=1}^{N}\left(\ell^{\prime}\left(\mathbf{x}_{n}^{\top}\mathbf{w}\left(t\right)\right)\right)^{2}} \, . \label{eq:normgradbatch}
\end{equation}
Next, we will rely on this key fact to prove our results for each case.

\subsection{Case 1: Random sampling with replacement}
From the $\beta$-smoothness of the loss
\[
\mathcal{L}\left(\mathbf{w}\left(t+1\right)\right)-\mathcal{L}\left(\mathbf{w}\left(t\right)\right)\leq\nabla\mathcal{L}\left(\mathbf{w}\right)^{\top}\left(\mathbf{w}\left(t+1\right)-\mathbf{w}\left(t\right)\right)+\frac{\beta\sigma_{\max}^2}{2}\left\Vert \mathbf{w}\left(t+1\right)-\mathbf{w}\left(t\right)\right\Vert ^{2}
\]

Taking expectation, we have
\begin{align*}
 & \E\mathcal{L}\left(\mathbf{w}\left(t+1\right)\right)-\mathcal{\E L}\left(\mathbf{w}\left(t\right)\right)\\
 & \leq\E\left[\nabla\mathcal{L}\left(\mathbf{w}\left(t\right)\right)^{\top}\left[\left(\mathbf{w}\left(t+1\right)-\mathbf{w}\left(t\right)\right)\right]\right]+\frac{\beta\sigma_{\max}^2}{2}\E\left\Vert \mathbf{w}\left(t+1\right)-\mathbf{w}\left(t\right)\right\Vert ^{2}\\
 & \overset{\left(1\right)}{\leq}\E\left[\E\left[\nabla\mathcal{L}\left(\mathbf{w}\left(t\right)\right)^{\top}\left(\mathbf{w}\left(t+1\right)-\mathbf{w}\left(t\right)\right)|\mathbf{w}\left(t\right)\right]\right]+\frac{\beta\sigma_{\max}^2}{2}\E\left\Vert \sum_{n=1}^{N}z_{t,n}\ell^{\prime}\left(\mathbf{w}\left(t\right)^{\top}\mathbf{x}_{n}\right)\mathbf{x}_{n}\right\Vert ^{2}\\
 & \overset{\left(2\right)}{\leq}-\eta\frac{1}{K}\E\left\Vert \nabla\mathcal{L}\left(\mathbf{w}\left(t\right)\right)\right\Vert ^{2}+\frac{\beta\sigma_{\max}^{4}}{2}\eta^{2}\E\left[\sum_{n=1}^{N}z_{t,n}^{2}\left(\ell^{\prime}\left(\mathbf{w}\left(t\right)^{\top}\mathbf{x}_{n}\right)\right)^{2}\right]\\
 & \overset{\left(3\right)}{\leq}-\eta\frac{1}{K}\E\left\Vert \nabla\mathcal{L}\left(\mathbf{w}\left(t\right)\right)\right\Vert ^{2}+\frac{\beta\sigma_{\max}^{4}}{2}\eta^{2}\frac{1}{K}\sum_{n=1}^{N}\E\left(\ell^{\prime}\left(\mathbf{w}\left(t\right)^{\top}\mathbf{x}_{n}\right)\right)^{2}\\
 & \overset{\left(4\right)}{\leq}-\eta\frac{1}{K}\E\left\Vert \nabla\mathcal{L}\left(\mathbf{w}\left(t\right)\right)\right\Vert ^{2}+\frac{\beta\sigma_{\max}^{4}}{2\gamma^{2}}\eta^{2}\frac{1}{K}\E\left\Vert \nabla\mathcal{L}\left(\mathbf{w}\left(t\right)\right)\right\Vert ^{2}\\
 & =-\frac{1}{K}\eta\left(1-\frac{\beta\sigma_{\max}^{4}}{2\gamma^{2}}\eta\right)\E\left\Vert \nabla\mathcal{L}\left(\mathbf{w}\left(t\right)\right)\right\Vert ^{2}\,,
\end{align*}
where in $\left(1\right)$ we defined $z_{t,n}$ as a random variable
equal to $1$ if sample $n$ is selected at time $t$, or $0$ otherwise,
in $\left(2\right)$ we used the definition of $\sigma_{\max}$, in
$\left(3\right)$ we used $\E z_{t,n}^{2}=\E z_{t,n}=K^{-1}$ and
$\E\left[\left(\mathbf{w}\left(t+1\right)-\mathbf{w}\left(t\right)\right)|\mathbf{w}\left(t\right)\right]=K^{-1}\nabla\mathcal{L}\left(\mathbf{w}\right)$,
and in $\left(4\right)$ we used eq. \ref{eq:normgradbatch}.
Therefore, if 
\begin{equation}
\eta<\frac{2\gamma^{2}}{\beta\sigma_{\max}^{4}}\label{eq: eta condition with replacement,appendix}
\end{equation}
then 
\[
q\triangleq\frac{1}{K}\left(1-\frac{\beta\sigma_{\max}^{4}}{2\gamma}\eta\right)>0\,,
\]
and we can write
\[
    \E\left\Vert \nabla\mathcal{L}\left(\mathbf{w}\left(t\right)\right)\right\Vert^{2} \le 
    \frac{\mathcal{\E L}\left(\mathbf{w}\left(t\right)\right)-\E\mathcal{L}\left(\mathbf{w}\left(t+1\right)\right)}{\eta q}.
\] Summing over $t$ we have
 
\begin{equation}
\sum_{t=1}^{\infty}\E\left\Vert \nabla\mathcal{L}\left(\mathbf{w}\left(t\right)\right)\right\Vert ^{2}\leq\frac{\E\mathcal{L}\left(\mathbf{w}\left(1\right)\right)-\lim_{t\rightarrow\infty}\mathcal{\E L}\left(\mathbf{w}\left(t\right)\right)}{\eta q}\leq\frac{\E\mathcal{L}\left(\mathbf{w}\left(1\right)\right)}{\eta q}<\infty\label{eq: inifinite sum of expected square gradients}
\end{equation}
and therefore $\E\left\Vert \nabla\mathcal{L}\left(\mathbf{w}\left(t\right)\right)\right\Vert \rightarrow0$.
Moreover, the Markov inequality, we have
\[
\mathrm{P}\left(\sum_{t=1}^{\infty}\left\Vert \nabla\mathcal{L}\left(\mathbf{w}\left(t\right)\right)\right\Vert ^{2}<c\right)\geq1-\frac{\E\sum_{t=1}^{\infty}\left\Vert \nabla\mathcal{L}\left(\mathbf{w}\left(t\right)\right)\right\Vert ^{2}}{c}\,.
\]
Combining this equation with equation \ref{eq: inifinite sum of expected square gradients}, and
taking the limit of $c$ to $\infty$, we obtain 
\begin{equation}
\mathrm{P}\left(\sum_{t=1}^{\infty}\left\Vert \nabla\mathcal{L}\left(\mathbf{w}\left(t\right)\right)\right\Vert ^{2}<\infty\right)=1\,.\label{eq: probability 1}
\end{equation}
Therefore, with probability 1, we have $\forall n:$ $\mathbf{w}\left(t\right)^{\top}\mathbf{x}_{n}\rightarrow\infty$,
which implies $\mathcal{L}\left(\mathbf{w}\left(t\right)\right)\rightarrow0$.
Moreover, 
\begin{align*}
 & \sum_{t=1}^{\infty}\left\Vert \mathbf{w}\left(t+1\right)-\mathbf{w}\left(t\right)\right\Vert ^{2}=\eta^{2}\sum_{t=1}^{\infty}\left\Vert \sum_{n\in\mathcal{B}\left(t\right)}\ell^{\prime}\left(\mathbf{w}\left(t\right)^{\top}\mathbf{x}_{n}\right)\mathbf{x}_{n}\right\Vert ^{2}\\
\leq & \eta^{2}\sigma_{\max}^{2}\sum_{t=1}^{\infty}\sum_{n\in\mathcal{B}\left(t\right)}\left(\ell^{\prime}\left(\mathbf{w}\left(t\right)^{\top}\mathbf{x}_{n}\right)\right)^{2}\leq\eta^{2}\sigma_{\max}^{2}\sum_{t=1}^{\infty}\sum_{n=1}^{N}\left(\ell^{\prime}\left(\mathbf{w}\left(t\right)^{\top}\mathbf{x}_{n}\right)\right)^{2}\\
\overset{\left(1\right)}{\leq} & \frac{\eta^{2}\sigma_{\max}^{2}}{\gamma^{2}}\sum_{t=1}^{\infty}\left\Vert \nabla\mathcal{L}\left(\mathbf{w}\left(t\right)\right)\right\Vert ^{2}\overset{\left(2\right)}{<}\infty
\end{align*}
where in $\left(1\right)$ we used eq. \ref{eq:normgradbatch}, and
$\left(2\right)$ is true with probability 1 from eq. \ref{eq: probability 1}.

\subsection{Case 2: Sampling without replacement}
Linear separability enforces a lower bound on the norm of these increments (eq. \ref{eq:normgradbatch}, which follows form Lemma \ref{lem: min NN PCA}). This bound enables us to bound the SGD increments, and other related quantities, in terms of the norm of the full gradient (Lemma \ref{lem: norm bounds} below).

\begin{lemma}
\label{lem: norm bounds}For all $t\in\mathbb{N}$ and $k\in\left\{ 0,1,\dots,K\right\} $, such that $t$ and $t+k-1$ are in the same epoch, we have
\begin{align*}
\left\Vert \mathbf{w}\left(t+k\right)-\mathbf{w}\left(t\right)+\eta\nabla\mathcal{L}\left(\mathbf{w}\left(t\right)\right)\right\Vert  & \leq \eta^{2}k\beta\sigma_{\max}^{3}\gamma^{-1}\left[1-\eta k\beta\sigma_{\max}^{2}\right]^{-1}\left\Vert \nabla\mathcal{L}\left(\mathbf{w}\left(t\right)\right)\right\Vert \\
\left\Vert \mathbf{w}\left(t+k\right)-\mathbf{w}\left(t\right)\right\Vert  & \leq
\eta\gamma^{-1}\sigma_{\max}\left[1-\eta k\beta\sigma_{\max}^{2}\right]^{-1}\left\Vert \nabla\mathcal{L}\left(\mathbf{w}\left(t\right)\right)\right\Vert \\
\left\Vert \nabla\mathcal{L}\left(\mathbf{w}\left(t+k\right)\right)-\nabla\mathcal{L}\left(\mathbf{w}\left(t\right)\right)\right\Vert & \leq\eta\beta\gamma^{-1}\sigma_{\max}^{2}\left[1-\eta k\beta\sigma_{\max}^{2}\right]^{-1}\left\Vert \nabla\mathcal{L}\left(\mathbf{w}\left(t\right)\right)\right\Vert \,.
\end{align*}
\end{lemma}

\begin{proof} See appendix section \ref{sec:proof of lemma 4}.
\end{proof}

Together, these bounds enable us to complete the proof. First, we assume that $t$ is the first iteration in some epoch, i.e., $t=uK$ for some $u\in\{0,1,2,\dots\}$. The $\beta$-smoothness of the loss function
$\ell\left(u\right)$ (Assumption \ref{assum: loss properties}),
implies that $\mathcal{L}\left(\mathbf{w}\left(t\right)\right)$ is
$\beta\sigma_{\max}^{2}$-smooth. This entails that
\begin{align}
 & \mathcal{L}\left(\mathbf{w}\left(t+K\right)\right)-\mathcal{L}\left(\mathbf{w}\left(t\right)\right)-\frac{\beta\sigma_{\max}^{2}}{2}\left\Vert \mathbf{w}\left(t+K\right)-\mathbf{w}\left(t\right)\right\Vert ^{2}\nonumber \\
 & \leq\nabla\mathcal{L}\left(\mathbf{w}\left(t\right)\right)^{\top}\left(\mathbf{w}\left(t+K\right)-\mathbf{w}\left(t\right)\right)\nonumber \\
 & =\nabla\mathcal{L}\left(\mathbf{w}\left(t\right)\right)^{\top}\left(-\eta\nabla\mathcal{L}\left(\mathbf{w}\left(t\right)\right)+\mathbf{w}\left(t+K\right)-\mathbf{w}\left(t\right)+\eta\nabla\mathcal{L}\left(\mathbf{w}\left(t\right)\right)\right)\nonumber \\
 & \leq-\eta\left\Vert \nabla\mathcal{L}\left(\mathbf{w}\left(t\right)\right)\right\Vert ^{2}+\left\Vert \nabla\mathcal{L}\left(\mathbf{w}\left(t\right)\right)\right\Vert \left\Vert \mathbf{w}\left(t+K\right)-\mathbf{w}\left(t\right)+\eta\nabla\mathcal{L}\left(\mathbf{w}\left(t\right)\right)\right\Vert \label{eq: smootness inequality}
\end{align}
and therefore, 
\begin{align*}
 & \mathcal{L}\left(\mathbf{w}\left(t+K\right)\right)-\mathcal{L}\left(\mathbf{w}\left(t\right)\right)\\
 & \overset{\left(1\right)}{\leq}-\eta\left\Vert \nabla\mathcal{L}\left(\mathbf{w}\left(t\right)\right)\right\Vert ^{2}+\eta^{2}K\beta\sigma_{\max}^{3}\gamma^{-1}\left[1-\eta K\beta\sigma_{\max}^{2}\right]^{-1}\left\Vert \nabla\mathcal{L}\left(\mathbf{w}\left(t\right)\right)\right\Vert ^2\\
 & +\frac{1}{2}\eta^2\beta\gamma^{-2}\sigma^4_{\max}\left[1-\eta K\beta\sigma_{\max}^{2}\right]^{-2}\left\Vert \nabla\mathcal{L}\left(\mathbf{w}\left(t\right)\right)\right\Vert^2\\
 & =-\eta\left(1-\eta\left(K\beta\sigma_{\max}^{3}\gamma^{-1}\left[1-\eta K\beta\sigma_{\max}^{2}\right]^{-1}+\frac{1}{2}\beta\gamma^{-2}\sigma_{\max}^{4}\left[1-\eta K\beta\sigma_{\max}^{2}\right]^{-2}\right)\right)\left\Vert \nabla\mathcal{L}\left(\mathbf{w}\left(t\right)\right)\right\Vert ^{2} \\
 & \overset{\left(2\right)}{\leq}-\eta\left(1-\eta 2\beta\sigma_{\max}^{3}\gamma^{-1}\left(K+\gamma^{-1}\sigma_{\max}\right)\right)\left\Vert \nabla\mathcal{L}\left(\mathbf{w}\left(t\right)\right)\right\Vert ^{2}\\
  & \overset{\left(3\right)}{=}-\eta\left(1-\eta q\right)\left\Vert \nabla\mathcal{L}\left(\mathbf{w}\left(t\right)\right)\right\Vert ^{2}
\end{align*}
where in $\left(1\right)$ we used eq. \ref{eq: smootness inequality}
and the first two equations in Lemma \ref{lem: norm bounds}, in $\left(2\right)$ we recall
we assumed that $\eta<1/(2K\beta\sigma_{\max}^{2})$ in eq. \ref{eq: eta condition}, and in $(3)$ we denoted $q=2\beta\sigma_{\max}^{3} \gamma^{-1}\left(K+\gamma^{-1}\sigma_{\max}\right)$.
Recall we assumed $\eta q<1$ in eq. \ref{eq: eta condition}. Summing $t$ over $0,K,2K,\dots,$we obtain 
\[
\sum_{u=0}^{\infty}\left\Vert \nabla\mathcal{L}\left(\mathbf{w}\left(uK\right)\right)\right\Vert ^{2}\leq\frac{\mathcal{L}\left(\mathbf{w}\left(0\right)\right)-\lim_{u\rightarrow\infty}\mathcal{L}\left(\mathbf{w}\left(uK\right)\right)}{\eta\left(1-\eta q\right)}\leq\frac{\mathcal{L}\left(\mathbf{w}\left(0\right)\right)}{\eta\left(1-\eta q\right)}<\infty
\]
since $\mathcal{L}\left(\mathbf{w}\right)\geq0$ and $\eta q < 1$ according to our assumption on $\eta$. 

Next, we consider general time $t$ (i.e., not only first iteration at epochs, as we assumed until now). We note that, for any $k$ such that $t+k-1$ is in the same epoch as $t$, we have that
\begin{align*}
\left\Vert \nabla\mathcal{L}\left(\mathbf{w}\left(t+k\right)\right)\right\Vert  & \leq \left\Vert \nabla\mathcal{L}\left(\mathbf{w}\left(t\right)\right)\right\Vert +\left\Vert \nabla\mathcal{L}\left(\mathbf{w}\left(t+k\right)\right)-\nabla\mathcal{L}\left(\mathbf{w}\left(t\right)\right)\right\Vert \\
 & \leq\left(1+\eta\beta\gamma^{-1}\sigma_{\max}^{2}\left[1-\eta k\beta\sigma_{\max}^{2}\right]^{-1}\right)\left\Vert \nabla\mathcal{L}\left(\mathbf{w}\left(t\right)\right)\right\Vert \, ,
\end{align*}
where we used the last equation in Lemma \ref{lem: norm bounds}. 
Thus, combining the last two equations we obtain 
\begin{align}
 & \sum_{u=0}^{\infty}\left\Vert \nabla\mathcal{L}\left(\mathbf{w}\left(u\right)\right)\right\Vert^2 =\sum_{u=0}^{\infty}\sum_{k=0}^{K-1}\left\Vert \nabla\mathcal{L}\left(\mathbf{w}\left(uK+k\right)\right)\right\Vert^2 \nonumber \\
\leq & \left(1+\eta\beta\gamma^{-1}\sigma_{\max}^{2}\left[1-\eta K\beta\sigma_{\max}^{2}\right]^{-1}\right)^{2}K\sum_{u=0}^{\infty}\left\Vert \nabla\mathcal{L}\left(\mathbf{w}\left(uK\right)\right)\right\Vert^2 <\infty \, \label{eq: summability of square gradients}
\end{align}
which also implies that $\left\Vert \nabla\mathcal{L}\left(\mathbf{w}\left(t\right)\right)\right\Vert \rightarrow0$.
Next, we recall eq. \ref{eq:normgradbatch} to obtain
\[
\sqrt{\sum_{n=1}^{N}\left(\ell^{\prime}\left(\mathbf{x}_{n}^{\top}\mathbf{w}\left(t\right)\right)\right)^{2}}\leq\frac{1}{\gamma}\left\Vert \nabla\mathcal{L}\left(\mathbf{w}\left(t\right)\right)\right\Vert \rightarrow0 \,.
\]
Therefore, $\forall n\,:\,\ell^{\prime}\left(\mathbf{x}_{n}^{\top}\mathbf{w}\left(t\right)\right)\rightarrow0$.
Since $(\ell^{\prime}\left(u\right))^2$ is strictly positive, and equal
to zero only at $u\rightarrow\infty$ (from assumption \ref{assum: loss properties}),
we obtain that 
$
\lim_{t\rightarrow\infty}\mathbf{x}_{n}^{\top}\mathbf{w}\left(t\right)=\infty\,.
$

Finally, using eq. \ref{eq:normgradbatch} again, we obtain
\begin{align}
 & \left\Vert \nabla\mathcal{L}\left(\mathbf{w}\left(t\right)\right)\right\Vert  \geq\gamma\sqrt{\sum_{n=1}^{N}\left(\ell^{\prime}\left(\mathbf{x}_{n}^{\top}\mathbf{w}\left(t\right)\right)\right)^{2}}\geq\gamma\sqrt{\sum_{n\in\mathcal{B}\left(t\right)}\left(\ell^{\prime}\left(\mathbf{x}_{n}^{\top}\mathbf{w}\left(t\right)\right)\right)^{2}} \nonumber \\
 & \geq\frac{\gamma}{\sigma_{\max}}\left\Vert \sum_{n\in\mathcal{B}\left(t\right)}\ell^{\prime}\left(\mathbf{x}_{n}^{\top}\mathbf{w}\left(t\right)\right)\mathbf{x}_{n}\right\Vert =\frac{\gamma}{\sigma_{\max}}\eta^{-1}\left\Vert \mathbf{w}\left(t+1\right)-\mathbf{w}\left(t\right)\right\Vert \,.\label{eq: norm sgd increment}
\end{align}

Combining eq. \ref{eq: norm sgd increment} and \ref{eq: summability of square gradients}
we obtain that 
$
\sum_{t=0}^{\infty}\left\Vert \mathbf{w}\left(t+1\right)-\mathbf{w}\left(t\right)\right\Vert ^{2}<\infty\,.
$
\QEDA

\subsection{Proof of Lemma \ref{lem: norm bounds} \label{sec:proof of lemma 4}}
First, we prove the following technical Lemma.
\begin{lemma}
\label{lem: delta recursion bound}Let $\epsilon$ and $\gamma$ be
two positive constants. If 
$
\delta_{k}\le \theta +\epsilon\sum_{u=0}^{k-1}\delta_{u},
$
then
\begin{equation}
\delta_{k}\leq\frac{\theta}{1-k\epsilon}\label{eq: recursion 1}
\end{equation}
 and 
\begin{equation}
\sum_{u=0}^{k-1}\delta_{u}\leq \frac{k\theta}{1-k\epsilon}\,.\label{eq: recursion 2}
\end{equation}
 \end{lemma}
\begin{proof}
We prove this by direct calculation 
\begin{align*}
\delta_{k} & \le\theta+\epsilon\sum_{u=0}^{k-1}\delta_{u}\leq\theta+\epsilon\sum_{u_{1}=0}^{k-1}\left(\theta+\epsilon\sum_{u_{2}=0}^{u_{1}-1}\delta_{u_{2}}\right)\\
 & \leq\theta+\epsilon\sum_{u_{1}=0}^{k-1}\theta+\epsilon^{2}\sum_{u_{1}=0}^{k-1}\sum_{u_{2}=0}^{u_{1}-1}\theta+\dots+\epsilon^{k}\sum_{u_{1}=0}^{k-1}\sum_{u_{2}=0}^{u_{1}-1}\cdots\sum_{u_{k}=0}^{u_{k-1}-1}\theta\\
 & \leq\theta\left[1+\epsilon k+\epsilon^{2}k\left(k-1\right)+\dots+\epsilon^{k}k!\right]\\
 & \leq \theta \sum_{u=0}^{k}(k\epsilon)^{u} = \theta \frac{1-(k\epsilon)^{k+1}}{1-k\epsilon} \leq  \frac{\theta}{1-k\epsilon} 
\end{align*}
Also, from the first and last lines in the above equation, we have
\[
\sum_{u=0}^{k-1}\delta_{u}\leq\theta\epsilon^{-1}\sum_{u=1}^{k}\left(k\epsilon\right)^{u}=\theta k\sum_{u=0}^{k-1}\left(k\epsilon\right)^{u}\leq\frac{k\theta}{1-k\epsilon}.
\]
\end{proof}
With this result in hand, we complete the proof by direct calculation 
\begin{align}
 & \left\Vert \mathbf{w}\left(t+k\right)-\mathbf{w}\left(t\right)+\eta\sum_{u=0}^{k-1}\sum_{n\in\mathcal{B}\left(t+u\right)}\ell^{\prime}\left(\mathbf{x}_{n}^{\top}\mathbf{w}\left(t\right)\right)\mathbf{x}_{n}\right\Vert \nonumber \\
 & =\left\Vert -\eta\sum_{u=0}^{k-1}\sum_{n\in\mathcal{B}\left(t+u\right)}  \left[ \ell^{\prime}\left(\mathbf{x}_{n}^{\top}\mathbf{w}\left(t+u\right)\right)-\ell^{\prime}\left(\mathbf{x}_{n}^{\top}\mathbf{w}\left(t\right)\right) \right] \mathbf{x}_{n}  \right\Vert \nonumber \\
  & \overset{\left(1\right)}{\leq} \eta\sum_{u=0}^{k-1}\left\Vert \sum_{n\in\mathcal{B}\left(t+u\right)}\left[-\ell^{\prime}\left(\mathbf{x}_{n}^{\top}\mathbf{w}\left(t+u\right)\right)+\ell^{\prime}\left(\mathbf{x}_{n}^{\top}\mathbf{w}\left(t\right)\right)\right]\mathbf{x}_{n}\right\Vert \nonumber \\
 & \overset{\left(2\right)}{\leq}\eta\sigma_{\max}\sum_{u=0}^{k-1}\sqrt{\sum_{n=1}^{N}\left[-\ell^{\prime}\left(\mathbf{x}_{n}^{\top}\mathbf{w}\left(t+u\right)\right)+\ell^{\prime}\left(\mathbf{x}_{n}^{\top}\mathbf{w}\left(t\right)\right)\right]^{2}}\nonumber \\
  & \overset{\left(3\right)}{\leq} \eta\beta\sigma_{\max}\sum_{u=0}^{k-1}\sqrt{\sum_{n=1}^{N}\left(\mathbf{x}_{n}^{\top}\left(\mathbf{w}\left(t+u\right)-\mathbf{w}\left(t\right)\right)\right)^{2}}\nonumber \\
 & \overset{\left(4\right)}{\leq} \eta\beta\sigma_{\max}^{2}\sum_{u=0}^{k-1}\left\Vert \left(\mathbf{w}\left(t+u\right)-\mathbf{w}\left(t\right)\right)\right\Vert \,,\label{eq: inequality Delta W plus Gradient}
\end{align}
where in $\left(1\right)$ we used the triangle inequality, in $(2)$ we define $\nu_{n}=-\ell^{\prime}\left(\mathbf{x}_{n}^{\top}\mathbf{w}\left(t+u\right)\right)+\ell^{\prime}\left(\mathbf{x}_{n}^{\top}\mathbf{w}\left(t\right)\right)$,
and used 
\[
\left\Vert \sum_{n\in\mathcal{B}\left(t+u\right)}\nu_{n}\mathbf{x}_{n}\right\Vert \leq\sigma_{\max}\sqrt{\sum_{n\in\mathcal{B}\left(t+u\right)}\nu_{n}^{2}}\leq\sigma_{\max}\sqrt{\sum_{n=1}^{N}\nu_{n}^{2}} \, ,
\]
in $\left(3\right)$ we used the fact that $\beta$ is the Lipshitz
constant of $\ell^{\prime}\left(u\right)$, and in $(4)$ we used the definition of $\sigma_{\mathrm{max}}$. The above bound implies the following bound
\begin{align}
 & \left\Vert \mathbf{w}\left(t+k\right)-\mathbf{w}\left(t\right)\right\Vert \nonumber \\
 &  \overset{\left(1\right)}{=} \left\Vert -\eta\sum_{u=0}^{k-1}\sum_{n\in\mathcal{B}\left(t+u\right)}\ell^{\prime}\left(\mathbf{x}_{n}^{\top}\mathbf{w}\left(t\right)\right)\mathbf{x}_{n}+\eta\sum_{u=0}^{k-1}\sum_{n\in\mathcal{B}\left(t+u\right)}\ell^{\prime}\left(\mathbf{x}_{n}^{\top}\mathbf{w}\left(t\right)\right)\mathbf{x}_{n}+\mathbf{w}\left(t+k\right)-\mathbf{w}\left(t\right)\right\Vert \nonumber \\
 & \overset{\left(2\right)}{\leq} \eta\left\Vert \sum_{u=0}^{k-1}\sum_{n\in\mathcal{B}\left(t+u\right)}\ell^{\prime}\left(\mathbf{x}_{n}^{\top}\mathbf{w}\left(t\right)\right)\mathbf{x}_{n}\right\Vert +\left\Vert \mathbf{w}\left(t+k\right)-\mathbf{w}\left(t\right)+\eta\sum_{u=0}^{k-1}\sum_{n\in\mathcal{B}\left(t+u\right)}\ell^{\prime}\left(\mathbf{x}_{n}^{\top}\mathbf{w}\left(t\right)\right)\mathbf{x}_{n}\right\Vert \nonumber \\
 & \overset{\left(3\right)}{\leq}\eta\gamma^{-1}\sigma_{\max}\left\Vert\nabla\mathcal{L}\left(\mathbf{w}\left(t\right)\right) \right\Vert +\eta\beta\sigma_{\max}^{2}\sum_{u=0}^{k-1}\left\Vert \mathbf{w}\left(t+u\right)-\mathbf{w}\left(t\right)\right\Vert \,, \label{eq: delta w inequality}
\end{align}
where in $(1)$ we added and subtracted the same term, in $(2)$ we used the triangle inequality, and in $\left(3\right)$ we used eq. \ref{eq: inequality Delta W plus Gradient} and also eq. \ref{eq:normgradbatch} to obtain
\begin{align}
& \left\Vert \sum_{u=0}^{k-1}\sum_{n\in\mathcal{B}\left(t+u\right)}\ell^{\prime}\left(\mathbf{x}_{n}^{\top}\mathbf{w}\left(t\right)\right)\mathbf{x}_{n}\right\Vert \leq\sigma_{\max}\sqrt{\sum_{u=0}^{k-1}\sum_{n\in\mathcal{B}\left(t+u\right)}\left(\ell^{\prime}\left(\mathbf{x}_{n}^{\top}\mathbf{w}\left(t\right)\right)\right)^{2}}  \nonumber\\
& \leq \sigma_{\max}\sqrt{\sum_{n=1}^{N}\left(\ell^{\prime}\left(\mathbf{x}_{n}^{\top}\mathbf{w}\left(t\right)\right)\right)^{2}}\leq\frac{\sigma_{\max}}{\gamma}\left\Vert \nabla\mathcal{L}\left(\mathbf{w}\left(t\right)\right)\right\Vert, \label{eq: X ni_k ineuqality}
\end{align}

Next, we apply eq. \ref{eq: recursion 1} from Lemma \ref{lem: delta recursion bound} on eq. \ref{eq: delta w inequality}, with $\delta_{k}=\left\Vert \mathbf{w}\left(t+k\right)-\mathbf{w}\left(t\right)\right\Vert ,$
$\epsilon=\eta\beta\sigma_{\max}^{2}$, and\\
$\theta=\eta\left(\sigma_{\max}/\gamma\right)\left\Vert \nabla\mathcal{L}\left(\mathbf{w}\left(t\right)\right)\right\Vert $ 
to obtain 
\begin{equation}
\left\Vert \mathbf{w}\left(t+k\right)-\mathbf{w}\left(t\right)\right\Vert {\leq}\eta\gamma^{-1}\sigma_{\max}\left[1-\eta k\beta\sigma_{\max}^{2}\right]^{-1}\left\Vert \nabla\mathcal{L}\left(\mathbf{w}\left(t\right)\right)\right\Vert  \label{eq: inequality Delta W norm}.
\end{equation}

Combining eqs. \ref{eq: inequality Delta W plus Gradient}, \ref{eq: inequality Delta W norm},
with eq. \ref{eq: recursion 2} implies 
\begin{align}
& \left\Vert \mathbf{w}\left(t+k\right)-\mathbf{w}\left(t\right)+\eta\sum_{u=0}^{k-1}\sum_{n\in\mathcal{B}\left(t+u\right)}\ell^{\prime}\left(\mathbf{x}_{n}^{\top}\mathbf{w}\left(t\right)\right)\mathbf{x}_{n}\right\Vert \\ 
& \leq \eta^{2}k\beta\sigma_{\max}^{3}\gamma^{-1}\left[1-\eta k\beta\sigma_{\max}^{2}\right]^{-1}\left\Vert \nabla\mathcal{L}\left(\mathbf{w}\left(t\right)\right)\right\Vert \,.\label{eq: inequality Delta W plus Gradient 2}
\end{align}

Finally, using eq. \ref{eq: inequality Delta W norm} we can directly prove the last part of the Lemma
\begin{align*}
 & \left\Vert \nabla\mathcal{L}\left(\mathbf{w}\left(t+k\right)\right)-\nabla\mathcal{L}\left(\mathbf{w}\left(t\right)\right)\right\Vert =\left\Vert \sum_{n=1}^{N}\ell^{\prime}\left(\mathbf{x}_{n}^{\top}\mathbf{w}\left(t+k\right)\right)-\sum_{n=1}^{N}\ell^{\prime}\left(\mathbf{x}_{n}^{\top}\mathbf{w}\left(t\right)\right)\right\Vert \\
\leq & \beta\left\Vert \sum_{n=1}^{N}\mathbf{x}_{n}^{\top}\left(\mathbf{w}\left(t+k\right)-\mathbf{w}\left(t\right)\right)\right\Vert \leq\beta\sigma_{\max}\left\Vert \mathbf{w}\left(t+k\right)-\mathbf{w}\left(t\right)\right\Vert \\
\leq & \eta\beta\gamma^{-1}\sigma_{\max}^{2}\left[1-\eta k\beta\sigma_{\max}^{2}\right]^{-1}\left\Vert \nabla\mathcal{L}\left(\mathbf{w}\left(t\right)\right)\right\Vert \,,
\end{align*}

Thus, we proved the Lemma, from the last equation, together with eqs. \ref{eq: inequality Delta W norm}
and \ref{eq: inequality Delta W plus Gradient 2}. \QEDA

\newpage
\section{Proof of Theorems \ref{thm: direction convergence to SVM} and \ref{thm: refined Theorem} \label{sec: proof of Theorems 2 and 3}}	
\subsection{Theorem \ref{thm: direction convergence to SVM} Proof} \label{sec: proof of theorem 2}
	In our proof we will use two auxiliary lemmata.
	
	\begin{restatable}{lemma}{sumSVAux}\label{lem:: sum sv aux lemma}

The following holds almost surely (with probability $1$) for random
sampling with replacement, and surely for sampling without replacement:
\begin{equation}
K\sum_{u=1}^{t-1}\frac{1}{u}\sum_{n\in\set\cap\mathcal{B}\left(u\right)}\alpha_{n}\mathbf{x}_{n}=\log\left(\frac{t}{K}\right)\hat{\mathbf{w}}+\check{\mathbf{w}}+\vect m_{1}(t),\label{eq: auxiliary calculation of SV part in w}
\end{equation}
where $\set$ is the set of indices of support vectors index, $\alpha_{n}$
are the SVM dual variables (so $\what=\sum_{n\in\set}\alpha_{n}\xn$),
$\check{\mathbf{w}}$ is some finite vector which is constant in $t$ (but can depend on the sample indices selected in the future), and $\forall$$\epsilon>0$,  $\vect m_{1}(t)$ is some vector such that
$\left\Vert \vect m_{1}\left(t\right)\right\Vert =o\left(t^{-0.5+\epsilon}\right)$,
and $\left\Vert \vect m_{1}\left(t+1\right)-\vect m_{1}\left(t\right)\right\Vert =O\left(t^{-1}\right).$ 
\end{restatable}


\noindent This Lemma is proved in section \ref{sec:Proof-of-Lemma}.
We define 
\begin{align}
\mathbf{r}\left(t\right)&=\mathbf{w}\left(t\right)-K\sum_{u=1}^{t-1}\frac{1}{u}\sum_{n\in\set\cap\mathcal{B}\left(u\right)}\alpha_{n}\mathbf{x}_{n}-\log\left(\eta\right)\what-\tilde{\mathbf{w}}\,-\check{\mathbf{w}}\nonumber\\
&\overset{(1)}{=}\mathbf{w}\left(t\right)-\log\left( \frac{t}{K} \right)\hat{\mathbf{w}}-\log\left(\eta\right)\what-\tilde{\mathbf{w}}-\vect m_{1}(t)\nonumber\\
&=\mathbf{w}\left(t\right)-\log\left( \frac{\eta}{K}t \right)\hat{\mathbf{w}}-\tilde{\mathbf{w}}-\vect m_{1}(t)
\,,\label{eq: r definition}
\end{align}
where the equality in (1) is true according to Lemma \ref{lem:: sum sv aux lemma}, and we define $\tilde{\mathbf{w}}$ as a vector that satisfies
\begin{equation}\label{eq: alpha_n definition}
\forall n\in\set: \alpha_{n}=\exp\left(-\tilde{\mathbf{w}}^{\top}\mathbf{x}_{n}\right)\,. 
\end{equation} 
Such a solution exists for almost every dataset, as a consequence of Lemma 12 in \cite{soudry2018journal}.
We denote the minimum margin to a non-support vector as:
\begin{equation}
\theta=\min_{n\notin\set}\mathbf{x}_{n}^{\top}\hat{\mathbf{w}}>1\,,\label{eq: theta definition}
\end{equation}
and by $C_{i}$,$\epsilon_{i}$,$t_{i}$ (\textbf{$i\in\mathbb{N}$})
various positive constants which are independent of $t$. 
Lastly, we define $\op\in\mathbb{R}^{d\times d}$ as the orthogonal projection matrix to the subspace spanned by the support vectors, and $\bar{\op}=\vect{I}-\op$ as the complementary projection.\\
\noindent The following Lemma is proved in
section \ref{sec:Auxiliary-Proof:}:
\begin{restatable}{lemma}{diff} \label{lem: (r(t+1)-r(t))r(t) bound}
$\exists\tilde{t},C_{2},C_{3}>0$ such that $\forall t>\tilde{t}$
\begin{equation}
\left(\mathbf{r}\left(t+1\right)-\mathbf{r}\left(t\right)\right)^{\top}\mathbf{r}\left(t\right)\le C_{2}t^{-\theta}+C_{3}t^{-1-0.5\tilde{\mu}}\,.\label{eq: (r(t+1)-r(t))r(t) bound}
\end{equation}
Additionally, $\forall \epsilon_1>0$, $\exists C_4, \tilde{t}_2$, such that $\forall t>\tilde{t}_2$, if
\begin{equation}\label{eq: Pr(t) lower bouded}
    \norm{\op \rvec\left(t\right)}\ge\epsilon_1
\end{equation}
then the following improved bound holds
\begin{equation}
\left(\mathbf{r}\left(t+1\right)-\mathbf{r}\left(t\right)\right)^{\top}\mathbf{r}\left(t\right)\le -C_{4}t^{-1}<0\,.\label{eq: (r(t+1)-r(t))r(t) improved bound}
\end{equation}
\end{restatable}

We note that
\begin{equation*}
	\boldsymbol{\rho}\left(t\right)=\mathbf{r}\left(t\right)+\tilde{\mathbf{w}}-\log(K)\what+\vect{m}_1(t)
\end{equation*}
and since $\forall \epsilon>0\ :\ \norm{\vect m_{1}(t)}=o\left( t^{-1+\epsilon}\right)$, using the triangle inequality we can write
\begin{equation*}
	\norm{\boldsymbol{\rho}\left(t\right)} \le \norm{\rvec(t)}+O(1).
\end{equation*}
\noindent Our goal is to show that $\left\Vert \mathbf{r}\left(t\right)\right\Vert $
is bounded, and therefore $\boldsymbol{\rho}\left(t\right)$
is bounded. 

\noindent We examine the equation
\begin{equation}
\Vert\rvec(t+1)\Vert^{2}=\Vert\rvec(t)\Vert^{2}+2\left(\rvec(t+1)-\rvec(t)\right)^{\top}\rvec(t)+\Vert\rvec(t+1)-\rvec(t)\Vert^{2}.\label{eq: r(t) diff equation}
\end{equation}
Using eq. \ref{eq: r definition} we can write 
\begin{align*}
& \left\Vert\rvec(t+1)-\rvec(t)\right\Vert^{2} \\ & =\left\Vert\mathbf{w}\left(t+1\right)-\log\left( \frac{t+1}{K} \right)\hat{\mathbf{w}}-\tilde{\mathbf{w}}-\vect m_{1}(t+1)-\left(\mathbf{w}\left(t\right)-\log\left( \frac{t}{K} \right)\hat{\mathbf{w}}-\tilde{\mathbf{w}}-\vect m_{1}(t)\right)\right\Vert^{2}\\
 & =\left\Vert\mathbf{w}\left(t+1\right)-\mathbf{w}\left(t\right)-\log\left(\frac{t+1}{t} \right)\hat{\mathbf{w}}-\vect m_{1}(t+1)+\vect m_{1}(t)\right\Vert^{2}.
\end{align*}
Since $\left\Vert \vect m_{1}\left(t+1\right)-\vect m_{1}\left(t\right)\right\Vert =O\left(t^{-1}\right)$  and $ \forall t>0\ :\ \log(1+t^{-1})\le t^{-1}$
we can write
\[
\Vert\rvec(t+1)-\rvec(t)\Vert^{2}=\Vert\mathbf{w}\left(t+1\right)-\mathbf{w}\left(t\right)+\vect a(t)\vect{\Vert^{2}},
\]
where $\vect a (t)\in\mathbb{R}^{n}$ and 
\begin{equation}
||\vect a (t)||=O\left(t^{-1}\right)\Rightarrow\forall\exists t_{1}\text{ such that }\forall t\ge t_{1}:\ ||\vect a (t)|| \le t^{-1}.\label{eq: f(t) is O(t^=00007B-1=00007D)}
\end{equation}
Thus, $\forall T\ge t_{1}$
\begin{align*}
 & \sum_{t=t_{1}}^{T}\Vert\rvec(t+1)-\rvec(t)\Vert^{2}=\sum_{t=t_{1}}^{T}\Vert\mathbf{w}\left(t+1\right)-\mathbf{w}\left(t\right)+\vect a(t)\Vert^{2}\\
 & =\sum_{t=t_{1}}^{T}\Vert\mathbf{w}\left(t+1\right)-\mathbf{w}\left(t\right)\Vert^{2}+2\sum_{t=t_{1}}^{T}\left(\mathbf{w}\left(t+1\right)-\mathbf{w}\left(t\right)\right)^{\top}\norm{\vect a\left(t\right)}+\sum_{t=t_{1}}^{T}\Vert\vect a\left(t\right)\Vert^{2}\\
 & \overset{(1)}{\le}\sum_{t=t_{1}}^{T}\Vert\mathbf{w}\left(t+1\right)-\mathbf{w}\left(t\right)\Vert^{2}+2\sqrt{\sum_{t=t_{1}}^{T}\Vert\mathbf{w}\left(t+1\right)-\mathbf{w}\left(t\right)\Vert^{2}\sum_{t=t_{1}}^{T}\norm{\vect a\left(t\right)}^{2}}+\sum_{t=t_{1}}^{T}\Vert\vect a\left(t\right)\Vert^{2}\\
 & \overset{(2)}{\le}\sum_{t=t_{1}}^{T}\Vert\mathbf{w}\left(t+1\right)-\mathbf{w}\left(t\right)\Vert^{2}+2\sqrt{\sum_{t=t_{1}}^{T}\Vert\mathbf{w}\left(t+1\right)-\mathbf{w}\left(t\right)\Vert^{2}\sum_{t=t_{1}}^{T}t^{-2}}+\sum_{t=t_{1}}^{T}t^{-2}\,,
\end{align*}

where in (1) we used Cauchy\textendash Schwarz inequality, and in
(2) we used eq. \ref{eq: f(t) is O(t^=00007B-1=00007D)}.

\noindent We take the limit $T\to\infty$ . Using the fact that $\forall v>1:\ \sum_{t=1}^{\infty}t^{-v}<\infty$
and $\sum_{t=1}^{\infty}\norm{\wvec(t+1)-\wvec(t)}^{2}<\infty$ from Theorem \ref{thm: Main Theorem 1},
we have that $\exists C_0$ such that
\begin{equation} \label{eq: square norm of r difference-1}
    \sum_{t=t_{1}}^{\infty}\Vert\rvec(t+1)-\rvec(t)\Vert^{2}=C_0<\infty.
\end{equation}
Note that this equation also implies that $\forall\epsilon_{0}$ 
\begin{equation}
\exists t_{0}:\forall t>t_{0}:\left|\left\Vert \mathbf{r}\left(t+1\right)\right\Vert -\left\Vert \mathbf{r}\left(t\right)\right\Vert \right|<\epsilon_{0}\,.\label{eq: norm difference convergence}
\end{equation}
Combining eqs. \ref{eq: (r(t+1)-r(t))r(t) bound}, \ref{eq: r(t) diff equation} and \ref{eq: square norm of r difference-1}, and using the fact that $\forall v>1:\ \sum_{t=1}^{\infty}t^{-v}<\infty$ we obtain
\begin{align*}
\norm{\rvec(t)}^{2}-\norm{\rvec(t_{1})}^{2} & =\sum_{u=t_{1}}^{t-1}\left[\norm{\rvec(u+1)}^{2}-\norm{\rvec(u)}^{2}\right]\\
 & =\sum_{u=t_{1}}^{t-1}\left[2\left(\rvec(u+1)-\rvec(u)\right)^{\top}\rvec(u)+\Vert\rvec(u+1)-\rvec(u)\Vert^{2}\right]<\infty
\end{align*}
and therefore $\rvec(t)$ is bounded.\QEDA

\subsection{Theorem \ref{thm: refined Theorem} Proof\label{sec: proof of theorem 3}}
From eqs. \ref{define wVec} and \ref{eq: r definition} we have that $\rhoVec=\wtilde+\rvec(t)+\vect m_{1}(t)$ where $\norm{\vect m_{1}(t)}\to0$. In order to show that $\lim\limits_{t\to\infty} \rhoVec=\wtilde$ we need to prove that $\norm{\rvec(t)}\to0$.\\
The proof of Theorem \ref{thm: refined Theorem} is identical to the proof of Theorem 4 in \cite{soudry2018journal}. We add the proof here for completeness.

In this proof, we need to show that $\left\Vert \mathbf{r}\left(t\right)\right\Vert \rightarrow0$
if $\mathrm{rank}\left(\mathbf{X}_{\set}\right)=\mathrm{rank}\left(\mathbf{X}\right)$,
and that $\tilde{\mathbf{w}}$ is unique given $\mathbf{w}\left(0\right)$.
To do so, this proof will continue where the proof of Theorem \ref{thm: direction convergence to SVM}
stopped, using notations and equations from that proof.

Since $\mathbf{r}\left(t\right)$ has a bounded norm, its two orthogonal
components $\mathbf{r}\left(t\right)=\op\mathbf{r}\left(t\right)+\bar{\op}\mathbf{r}\left(t\right)$
also have bounded norms (recall that $\op,\bar{\op}$
were defined in the beginning of appendix section \ref{sec: proof of theorem 2}).
From eq. \ref{eq: gradient descent linear}, $\forall t:\ \sum_{n\in\mathcal{B}\left(t\right)}\ell^{\prime}\left(\mathbf{w}\left(t\right)^{\top}\mathbf{x}_{n}\right)\mathbf{x}_{n}$
is spanned by the columns of $\mathbf{X}$. If $\mathrm{rank}\left(\mathbf{X}_{\set}\right)=\mathrm{rank}\left(\mathbf{X}\right)$,
then it is also spanned by the columns of $\mathbf{X}_{\set}$, and
so $\forall t:\ \sum_{n\in\mathcal{B}\left(t\right)}\ell^{\prime}\left(\mathbf{w}\left(t\right)^{\top}\bar{\op}\mathbf{x}_{n}\right)\mathbf{x}_{n}=0$.
Therefore, $\bar{\op}\mathbf{r}\left(t\right)$ is not
updated during SGD, and remains constant. Since $\tilde{\mathbf{w}}$
in eq. \ref{eq: r definition} is also bounded, we can absorb this
constant $\bar{\op}\mathbf{r}\left(t\right)$ into \textbf{$\tilde{\mathbf{w}}$}
without affecting eq. \ref{eq: w tilde} (since $\forall n\in\set:\,\mathbf{x}_{n}^{\top}\bar{\op}\mathbf{r}\left(t\right)=0$).
Thus, without loss of generality, we can assume that $\mathbf{r}\left(t\right)=\op\mathbf{r}\left(t\right)$.

We define the set 
\[
\mathcal{T}=\left\{ t>\max\left[t_{2},t_{0}\right]:\ensuremath{\left\Vert \mathbf{r}\left(t\right)\right\Vert <\epsilon_{1}}\right\} \,.
\]
By contradiction, we assume that the complementary set is not finite,
\[
\bar{\mathcal{T}}=\left\{ t>\max\left[t_{2},t_{0}\right]:\ensuremath{\left\Vert \mathbf{r}\left(t\right)\right\Vert \geq\epsilon_{1}}\right\} \,.
\]
Additionally, the set $\mathcal{T}$ is not finite: if it were finite, it would
have had a finite maximal point $t_{\max}\in\mathcal{T}$, and then, combining eqs. \ref{eq: (r(t+1)-r(t))r(t) improved bound}, \ref{eq: r(t) diff equation}, and \ref{eq: square norm of r difference-1},
we would find that $\forall t>t_{\max}$
\begin{align*}
\left\Vert \mathbf{r}\left(t\right)\right\Vert ^{2}-\left\Vert \mathbf{r}\left(t_{\max}\right)\right\Vert ^{2} & =\sum_{u=t_{\max}}^{t-1}\left[\left\Vert \mathbf{r}\left(u+1\right)\right\Vert ^{2}-\left\Vert \mathbf{r}\left(u\right)\right\Vert ^{2}\right]\leq C_{0}-2C_{2}\sum_{u=t_{\max}}^{t-1}u^{-1}\rightarrow-\infty\,,
\end{align*}
which is impossible since $\left\Vert \mathbf{r}\left(t\right)\right\Vert ^{2}\geq0$.
Furthermore, eq. \ref{eq: square norm of r difference-1}
implies that 
\[
\sum_{u=0}^{t}\left\Vert \mathbf{r}\left(u+1\right)-\mathbf{r}\left(t\right)\right\Vert ^{2}=C_{0}-h\left(t\right)
\]
 where $h\left(t\right)$ is a positive monotone function decreasing
to zero. Let $t_{3},t$ be any two points such that $t_{3}<t$,
$\left\{ t_{3}, t_{3}+1,\dots t\right\} \subset\bar{\mathcal{T}}$, and $\left(t_{3}-1\right)\in\mathcal{T}$. For all such  $t_{3}$ and $t$, we have 
\begin{align}
\left\Vert \mathbf{r}\left(t\right)\right\Vert ^{2}\nonumber & \leq\left\Vert \mathbf{r}\left(t_{3}\right)\right\Vert ^{2}+\sum_{u=t_{3}}^{t-1}\left[\left\Vert \mathbf{r}\left(u+1\right)\right\Vert ^{2}-\left\Vert \mathbf{r}\left(u\right)\right\Vert ^{2}\right]\\ 
\nonumber &= \left\Vert \mathbf{r}\left(t_{3}\right)\right\Vert ^{2}+\sum_{u=t_{3}}^{t-1}\left[\left\Vert \mathbf{r}\left(u+1\right)-\mathbf{r}\left(u\right)\right\Vert ^{2}+2\left(\mathbf{r}\left(u+1\right)-\mathbf{r}\left(u\right)\right)^{\top}\mathbf{r}\left(u\right)\right]\\
\nonumber
 & \leq\left\Vert \mathbf{r}\left(t_{3}\right)\right\Vert ^{2}+h\left(t_{3}\right)-h\left(t-1\right)-2C_{2}\sum_{u=t_{3}}^{t-1}u^{-1}\\ 
 & \leq\left\Vert \mathbf{r}\left(t_{3}\right)\right\Vert ^{2}+h\left(t_{3}\right)\,. \label{eq: r(t) bound}
\end{align}
Also, recall that $t_{3}>t_{0}$, so from eq. \ref{eq: norm difference convergence},
we have that $\left|\left\Vert \mathbf{r}\left(t_{3}\right)\right\Vert -\left\Vert \mathbf{r}\left(t_{3}-1\right)\right\Vert \right|<\epsilon_{0}$.
Since $\left\Vert \mathbf{r}\left(t_{3}-1\right)\right\Vert <\epsilon_{1}$ (from $\mathcal{T}$ definition),
we conclude that $\left\Vert \mathbf{r}\left(t_{3}\right)\right\Vert \leq\epsilon_{1}+\epsilon_{0}$.
Moreover, since $\mathcal{\bar{\mathcal{T}}}$ is an infinite set,
we can choose $t_{3}$ as large as we want. This implies that $\forall\epsilon_{2}>0$
we can find $t_{3}$ such that $\epsilon_{2}>h\left(t_{3}\right)$,
since $h\left(t\right)$ is a monotonically decreasing function. Therefore, from eq. \ref{eq: r(t) bound}, 
$\forall\epsilon_{1},\epsilon_{0},\epsilon_{2}$, $\exists t_{3}\in \bar{\mathcal{T}}$
such that 
\[
\forall t>t_{3}:\,\left\Vert \mathbf{r}\left(t\right)\right\Vert ^{2}\leq\epsilon_{1}+\epsilon_{0}+\epsilon_{2}\,.
\]
This implies that $\left\Vert \mathbf{r}\left(t\right)\right\Vert \rightarrow0$.

Lastly, we note that since $\bar{\op}\mathbf{r}\left(t\right)$
is not updated during SGD, we have that $\bar{\op}\left(\tilde{\mathbf{w}}-\mathbf{w}\left(0\right)\right)=0$.
This sets $\tilde{\mathbf{w}}$ uniquely, together with eq. \ref{eq: w tilde}.
$\blacksquare$

\subsection{Proof of Lemma \ref{lem: (r(t+1)-r(t))r(t) bound}\label{sec:Auxiliary-Proof:}}
	\diff*
    We focus on functions with exponential tail (definition \ref{def: exponential tail}):
\begin{equation}
(1-\e(-\mu_{-}u))e^{-u}\le-\ell'(u)\le(1+\e(-\mu_{+}u))e^{-u}\label{eq: exp tail def}
\end{equation}
\noindent Eq. \ref{eq: r definition} ($\rvec(t)$ definition) implies that 
\begin{align}
\mathbf{r}\left(t+1\right)-\mathbf{r}\left(t\right) & =\mathbf{w}\left(t+1\right)-\mathbf{w}\left(t\right)-K\sum_{u=1}^{t}\frac{1}{u}\sum_{n\in\set\cap\mathcal{B}\left(u\right)}\alpha_{n}\mathbf{x}_{n}+K\sum_{u=1}^{t-1}\frac{1}{u}\sum_{n\in\set\cap\mathcal{B}\left(u\right)}\alpha_{n}\mathbf{x}_{n}\label{eq: r dot}\nonumber \\
 & =-\eta\sum_{n\in\mathcal{B}\left(t\right)}\ell'\left(-\mathbf{x}_{n}^{\top}\mathbf{w}\left(t\right)\right)\mathbf{x}_{n}-\frac{K}{t}\sum_{n\in\set\cap\mathcal{B}\left(t\right)}\alpha_{n}\mathbf{x}_{n}.
\end{align}
Therefore,
\begin{align}
 & \left(\mathbf{r}\left(t+1\right)-\mathbf{r}\left(t\right)\right)^{\top}\mathbf{r}\left(t\right)\nonumber \\
 & \overset{(1)}{=}-\eta\sum_{n\in\mathcal{B}\left(t\right)}\ell'\left(-\mathbf{x}_{n}^{\top}\mathbf{w}\left(t\right)\right)\mathbf{x}_{n}^{\top}\mathbf{r}\left(t\right)-\frac{K}{t}\sum_{n\in\set\cap\mathcal{B}\left(t\right)}\alpha_{n}\xn^{\top}\mathbf{r}\left(t\right)\nonumber \\
 & \overset{(2)}{=}-\eta\sum_{n\in\mathcal{B}\left(t\right)\setminus\set}\ell'\left(-\log\left(\frac{\eta}{K}t \right)\hat{\mathbf{w}}^{\top}\mathbf{x}_{n}+\tilde{m}_n(t)-\wtilde^{\top}\xn-\mathbf{x}_{n}^{\top}\mathbf{r}\left(t\right)\right)\mathbf{x}_{n}^{\top}\mathbf{r}\left(t\right)\nonumber \\
 & -\eta\sum_{n\in\set\cap\mathcal{B}\left(t\right)}\ell'\left(-\log\left( \frac{\eta}{K}t \right)\hat{\mathbf{w}}^{\top}\mathbf{x}_{n}+\tilde{m}_n(t)-\wtilde^{\top}\xn-\mathbf{x}_{n}^{\top}\mathbf{r}\left(t\right)\right)\mathbf{x}_{n}^{\top}\mathbf{r}\left(t\right)-\frac{K}{t}\sum_{n\in\set\cap\mathcal{B}\left(t\right)}\alpha_{n}\mathbf{x}_{n}^{\top}\mathbf{r}\left(t\right),\label{eq: (r(t+1)-r(t))r(t)}
\end{align}
where in $(1)$ we used eq. \ref{eq: r dot}, in $(2)$ we used eq. \ref{eq: r definition} 
($\rvec(t)$ definition) and defined $\tilde{m}_n(t)=-\xn^\top\vect m_{1}(t)$. We note that $\forall n,\forall\epsilon>0$,  $\vert\tilde{m}_n(t)\vert=o\left(t^{-0.5+\epsilon}\right).$ \\
We examine the two parts of equation \ref{eq: (r(t+1)-r(t))r(t)}.
The first term is
\begin{align} \label{eq: non sv upper bound}
 & -\eta\sum_{n\in\mathcal{B}\left(t\right)\setminus\set}\ell'\left(-\log\left( \frac{\eta}{K}t \right)\hat{\mathbf{w}}^{\top}\mathbf{x}_{n}+\tilde{m}_n(t)-\wtilde^{\top}\xn-\mathbf{x}_{n}^{\top}\mathbf{r}\left(t\right)\right)\mathbf{x}_{n}^{\top}\mathbf{r}\left(t\right)\nonumber\\
 & \overset{(1)}{=}-\eta\sum_{\substack{n\in\mathcal{B}\left(t\right)\setminus\set\nonumber\\
x_{n}\mathbf{r}\left(t\right)\ge0
}
}\ell'\left(-\log\left( \frac{t\eta}{K}t \right)\hat{\mathbf{w}}^{\top}\mathbf{x}_{n}+\tilde{m}_n(t)-\wtilde^{\top}\xn-\mathbf{x}_{n}^{\top}\mathbf{r}\left(t\right)\right)\mathbf{x}_{n}^{\top}\mathbf{r}\left(t\right)\nonumber\\
 & \overset{(2)}{\le}\eta\sum_{\substack{n\in\mathcal{B}\left(t\right)\setminus\set\nonumber\\
x_{n}\mathbf{r}\left(t\right)\ge0
}
}\left(1+\e\left(-\mu_{+}\mathbf{x}_{n}^{\top}\mathbf{w}\left(t\right)\right)\right)\exp\left(-\log\left( \frac{\eta}{K}t \right)\hat{\mathbf{w}}^{\top}\mathbf{x}_{n}+\tilde{m}_n(t)-\wtilde^{\top}\xn-\mathbf{x}_{n}^{\top}\mathbf{r}\left(t\right)\right)\mathbf{x}_{n}^{\top}\mathbf{r}\left(t\right)\nonumber\\
 & \overset{(3)}{\le} \eta\sum_{\substack{n\in\mathcal{B}\left(t\right)\setminus\set\nonumber\\
x_{n}\mathbf{r}\left(t\right)\ge0
}
}2\alpha_{n}\exp\left(-\log\left( \frac{\eta}{K}t \right)\hat{\mathbf{w}}^{\top}\mathbf{x}_{n}+\tilde{m}_n(t)-\mathbf{x}_{n}^{\top}\mathbf{r}\left(t\right)\right)\mathbf{x}_{n}^{\top}\mathbf{r}\left(t\right)\nonumber\\
 & \overset{(4)}{\le}\eta\sum_{\substack{n\in\mathcal{B}\left(t\right)\setminus\set\nonumber\\
x_{n}\mathbf{r}\left(t\right)\ge0
}
}2\alpha_{n}\left( \frac{\eta}{K}t \right)^{-\hat{\mathbf{w}}^{\top}\mathbf{x}_{n}}\exp\left(\tilde{m}_n(t)\right)\nonumber\\
 & \overset{(5)}{\le}4\eta N\left(\max_{n}\alpha_{n}\right)\left(\frac{\eta}{K}t \right)^{-\theta}= 4N\left(\max_{n}\alpha_{n}\right)\eta^{-(\theta-1)}K^{\theta}t^{-\theta},\ \forall t>t_{2},
\end{align}
where in $(1)$ we used $-\ell'(u)\ge0$, in $(2)$ we used eq. \ref{eq: exp tail def},
in $(3)$ we used  $\alpha_{n}=\exp\left(-\tilde{\mathbf{w}}^{\top}\mathbf{x}_{n}\right)$ 
(eq. \ref{eq: alpha_n definition}) and the fact that $\exists t_{1}>0$ 
so that $\forall t>t_{1\ :\ }1+\e\left(-\mu_{+}\mathbf{x}_{n}^{\top}\mathbf{w}\left(t\right)\right)\le2$ 
since $\lim_{t\to\infty}\xn^{\top}\wvec(t)=\infty$. In $(4)$ we used
the relation $\forall x\ge0\ :\ xe^{-x}\le1$, in $(5)$ we used $\theta=\min_{n\not\in\set}\what^{\top}\xn>1$ (eq. \ref{eq: theta definition}) 
and the fact that $\exists t_{2}>t_{1}$ so that $\forall t>t_{2\ :\ }\e\left(\tilde{m}_n(t)\right)\le2$ 
since $\lim_{t\to\infty}\tilde{m}_n(t)=0$. \\
We define
\[
\gamma_{n}(t)=\begin{cases}
(1+\e(-\mu_{+}\xn^{\top}\wvec(t))) & \mathbf{x}_{n}^{\top}\mathbf{r}\left(t\right)\ge0\\
(1-\e(-\mu_{-}\xn^{\top}\wvec(t))) & \mathbf{x}_{n}^{\top}\mathbf{r}\left(t\right)<0
\end{cases}
\]
Using this definition, the second term in eq. \ref{eq: (r(t+1)-r(t))r(t)}
is
\begin{align*}
 & -\eta\sum_{n\in\set\cap\mathcal{B}\left(t\right)}\ell'\left(-\log\left( \frac{\eta}{K}t \right)\hat{\mathbf{w}}^{\top}\mathbf{x}_{n}+\tilde{m}_n(t)-\wtilde^{\top}\xn-\mathbf{x}_{n}^{\top}\mathbf{r}\left(t\right)\right)\mathbf{x}_{n}^{\top}\mathbf{r}\left(t\right)-\frac{K}{t}\sum_{n\in\set\cap\mathcal{B}\left(t\right)}\alpha_{n}\mathbf{x}_{n}^{\top}\mathbf{r}\left(t\right)\\
 & \overset{(1)}{\le}\eta\sum_{n\in\set\cap\mathcal{B}\left(t\right)}\gamma_{n}(t)\exp\left(-\log\left( \frac{\eta}{K}t \right)\hat{\mathbf{w}}^{\top}\mathbf{x}_{n}+\tilde{m}_n(t)-\wtilde^{\top}\xn-\mathbf{x}_{n}^{\top}\mathbf{r}\left(t\right)\right)\mathbf{x}_{n}^{\top}\mathbf{r}\left(t\right)-\frac{K}{t}\sum_{n\in\set\cap\mathcal{B}\left(t\right)}\alpha_{n}\mathbf{x}_{n}^{\top}\mathbf{r}\left(t\right)\\
 & \overset{(2)}{=}\sum_{n\in\set\cap\mathcal{B}\left(t\right)}\gamma_{n}(t)\alpha_{n}\frac{K}{t}\exp\left(\tilde{m}_n(t)-\mathbf{x}_{n}^{\top}\mathbf{r}\left(t\right)\right)\mathbf{x}_{n}^{\top}\mathbf{r}\left(t\right)-\frac{K}{t}\sum_{n\in\set\cap\mathcal{B}\left(t\right)}\alpha_{n}\mathbf{x}_{n}^{\top}\mathbf{r}\left(t\right)\\
 & =\sum_{n\in\set\cap\mathcal{B}\left(t\right)}\alpha_{n}\frac{K}{t}\left(\gamma_{n}(t)\exp\left(\tilde{m}_n(t)-\mathbf{x}_{n}^{\top}\mathbf{r}\left(t\right)\right)-1\right)\mathbf{x}_{n}^{\top}\mathbf{r}\left(t\right)
\end{align*}
where in $(1)$ we used eq. \ref{eq: exp tail def} and in $(2)$ we used $\forall n\in\set:\ \what^\top\xn=1$ and $\alpha_{n}=\eta\exp\left(-\tilde{\mathbf{w}}^{\top}\mathbf{x}_{n}\right)$
(eq. \ref{eq: alpha_n definition}).\\
We denote $\tilde{\mu}=\min\left(\mu_{+},\mu_{-},0.5\right)$. Recalling that $\forall n,\forall\epsilon>0$: $\vert \tilde{m}_n(t)\vert=o\left(t^{-1+\epsilon}\right)$ we note that this implies
\begin{equation} \label{eq: m_n(t)=o(t^{-0.5mu})}
	\forall n:\ \vert \tilde{m}_n(t)\vert =o(t^{-0.5\tilde{\mu}}).
\end{equation}
We examine each term $n$ in the sum:
\begin{equation}
\alpha_{n}\frac{K}{t}\left(\gamma_{n}(t)\exp\left(\tilde{m}_n(t)-\mathbf{x}_{n}^{\top}\mathbf{r}\left(t\right)\right)-1\right)\mathbf{x}_{n}^{\top}\mathbf{r}\left(t\right)\label{eq: individual term sv}
\end{equation}
and divide into cases.
\begin{enumerate}
    \item If $\vert\vect x_{n}^{\top}\rvec(t)\vert\le C_{1}t^{-0.5\tilde{\mu}}$
    then eq. \ref{eq: individual term sv} can be upper bounded by
    \begin{equation} \label{eq: sv term upper bound when xr is bounded}
        \left(\max_{n}\alpha_{n}\right)4KC_{1}t^{-1-0.5\tilde{\mu}},\forall t>t_{3}
    \end{equation}
    where we used $\vert \tilde{m}_n(t)\vert\to 0\Rightarrow\exists t_{3}>t_2$
    so that $\forall t>t_{3}:\ \exp\left(\tilde{m}_n(t)\right)\le2$.
    \item If $\vert\vect x_{k}^{\top}\rvec(t)\vert>C_{1}t^{-0.5\tilde{\mu}}$
    and $\vect x_{n}^{\top}\rvec(t)\ge0$ then
    \begin{align*}
    \e(-\mu_{+}\xn^{\top}\wvec(t)) & =\e\left(\mu_{+}\left(-\log\left( \frac{t}{K} \right)+\tilde{m}_n(t)-\wtilde^{\top}\xn-\mathbf{x}_{n}^{\top}\mathbf{r}\left(t\right)\right)\right)\\
     & \le\exp\left(-\mu_{+}\wtilde^{\top}\xn\right)\left(\frac{K}{t}\exp\left(\tilde{m}_n(t)\right)\right)^{\mu_{+}}\\
     & \le\left(2K\right)^{\mu_{+}}\exp\left(-\mu_{+}\min_{n}\wtilde^{\top}\xn\right)t^{-\mu_{+}}\triangleq C_{4}t^{-\mu_{+}}, \forall t>t_3.
    \end{align*}
    Using the last equation, eq. \ref{eq: exp tail def} and the fact that
    $\forall x\le1\ :\ e^{x}\le1+x+x^{2}$, eq. \ref{eq: individual term sv} can be upper bounded by
    \begin{align*}
     & \alpha_{n}\frac{K}{t}\left((1+\e(-\mu_{+}\xn^{\top}\wvec(t)))\exp(\tilde{m}_n(t))\exp(-C_{1}t^{-0.5\tilde{\mu}})-1\right)\mathbf{x}_{n}^{\top}\mathbf{r}\left(t\right)\\
     & \le\alpha_{n}\frac{K}{t}\left((1+C_{4}t^{-\mu_{+}})\left(1+\tilde{m}_n(t)+\tilde{m}^{2}_n(t)\right)\left(1-C_{1}t^{-0.5\tilde{\mu}}+C_{1}^{2}t^{-\tilde{\mu}}\right)-1\right)\mathbf{x}_{n}^{\top}\mathbf{r}\left(t\right)\\
     & \overset{(1)}{\le}\alpha_{n}\frac{K}{t}\left(-C_{1}t^{-0.5\tilde{\mu}}+o(t^{-0.5\tilde{\mu}})\right)\mathbf{x}_{n}^{\top}\mathbf{r}\left(t\right)\overset{(2)}{<}0,\forall t>t_{+}
    \end{align*}
    where in $(1)$ we used $\tilde{\mu}$ definition and eq. \ref{eq: m_n(t)=o(t^{-0.5mu})}. In $(2)$ we used that $-C_{1}t^{-0.5\tilde{\mu}}$ decrease to zero
    slower than the other terms and therefore $\exists t_{+}>t_{3}$ such
    that $\forall t>t_{+}$the last equation is negative.
    \item If $\vect x_{k}^{\top}\rvec(t)\ge\epsilon_2$ then $\exists t'+\ge t_+$ so that
    \begin{align*}
        &\gamma_{n}(t)\exp\left(\tilde{m}_n(t)-\mathbf{x}_{n}^{\top}\mathbf{r}\left(t\right)\right)\\
        &=
        (1+\e(-\mu_{+}\xn^{\top}\wvec(t)))\exp\left(\tilde{m}_n(t)-\mathbf{x}_{n}^{\top}\mathbf{r}\left(t\right)\right)\\
        &\le
        (1+\e(-\mu_{+}\xn^{\top}\wvec(t)))\exp\left(\tilde{m}_n(t)-\epsilon_2\right)\\
        & \le \exp\left(-0.5\epsilon_2\right)\,,
    \end{align*}
    where in the last transition we used the fact that $\tilde{m}_n(t)\to0$ and $\xn^{\top}\wvec(t)\to\infty$. Using this result, eq. \ref{eq: individual term sv} can be upper bounded $\forall t\ge t'_+$ by
    \begin{equation} \label{eq: sv term positive improved bound}
        -\min_n\alpha_{n}\frac{K}{t}\left(1- \exp\left(-0.5\epsilon_2\right)\right)\epsilon_2\triangleq - C_{+}''t^{-1}\,.
    \end{equation}
    where we defined $C_{+}''=\min_n\alpha_{n}K\left(1- \exp\left(-0.5\epsilon_2\right)\right)\epsilon_2$.
    \item If $\vert\vect x_{k}^{\top}\rvec(t)\vert>C_{1}t^{-0.5\tilde{\mu}}$
    and $\vect x_{n}^{\top}\rvec(t)<0$ then eq. \ref{eq: individual term sv}
    can be upper bounded by
    \begin{align*}
     & \alpha_{n}\frac{K}{t}\left(1-(1-\e(-\mu_{-}\xn^{\top}\wvec(t)))\exp\left(\tilde{m}_n(t)-\mathbf{x}_{n}^{\top}\mathbf{r}\left(t\right)\right)\right)\abs{\mathbf{x}_{n}^{\top}\mathbf{r}\left(t\right)}\,.
    \end{align*}
    We will now show that this equation is negative for sufficiently large
    $t$. We need to show that
    \[
    \left(1-\e\left(-\mu_{-}\xn^{\top}\wvec(t)\right)\right)\exp\left(\tilde{m}_n(t)-\mathbf{x}_{n}^{\top}\mathbf{r}\left(t\right)\right)>1
    \]
    Let $M>1$ be some arbitrary constant. We note that since $\lim_{t\to\infty}\xn^{\top}\wvec(t)=\infty$,
    $\exists t_{M}$ so that $\forall t>t_{M}:\ 1-\e\left(-\mu_{-}\xn^{\top}\wvec(t)\right)>0$.
    \\
    $\forall t>t_{M}$ if $\exp\left(-\xn^{\top}\rvec(t)\right)\ge M>1$
    then since $\lim_{t\to\infty}\xn^{\top}\wvec(t)=\infty$, 
    $\lim_{t\to\infty}\tilde{m}_n(t)=0$, $\exists t_{1}^{-}>t_{M}$ so that $\forall t>t_{1}^{-}$
    \begin{align*}
    	&\left(1-\e\left(-\mu_{-}\xn^{\top}\wvec(t)\right)\right)\exp\left(\tilde{m}_n(t)-\mathbf{x}_{n}^{\top}\mathbf{r}\left(t\right)\right)\\
        &\ge M\left(1-\e\left(-\mu_{-}\xn^{\top}\wvec(t)\right)\right)\exp\left(\tilde{m}_n(t)\right)\ge M'>1.
    \end{align*}
    
    In addition, if $\exists t>t_{M}$ so that $\exp\left(-\xn^{\top}\rvec(t)\right)<M$
    then
    \begin{align*}
     & \exp\left(-\mathbf{x}_{n}^{\top}\mathbf{r}\left(t\right)\right)\left(1-\e\left(-\mu_{-}\xn^{\top}\wvec(t)\right)\right)\exp\left(\tilde{m}_n(t)\right)\\
     & \overset{(1)}{=}\exp\left(-\mathbf{x}_{n}^{\top}\mathbf{r}\left(t\right)\right)\left(1-\e\left(\mu_{-}\left[-\log\left( \frac{t}{K} \right)+\tilde{m}_n(t)-\wtilde^{\top}\xn-\mathbf{x}_{n}^{\top}\mathbf{r}\left(t\right)\right]\right)\right)\exp\left(m_n(t)\right)\\
     & \overset{(2)}{\ge}\exp\left(-\mathbf{x}_{n}^{\top}\mathbf{r}\left(t\right)\right)\left(1-M^{\mu_{-}}\left( \frac{t}{K} \right)^{-\mu_{-}}\e\left(\mu_{-}\left[\tilde{m}_n(t)-\wtilde^{\top}\xn\right]\right)\right)\exp\left(\tilde{m}_n(t)\right)\\
     &  \overset{(3)}{\ge}\left(1+C_{1}t^{-0.5\tilde{\mu}}\right)\left(1-C't^{-\mu_{-}}\right)\left(1+\tilde{m}_n(t)\right)\\
     & =1+C_{1}t^{-0.5\tilde{\mu}}+o\left(t^{-0.5\tilde{\mu}}\right)\overset{(4)}{>}1,\ \forall t>t_{-},
    \end{align*}
    where in $(1)$ we used eq. \ref{eq: r definition} ($\rvec(t)$ definition),
    in $(2)$ we used $\exp\left(-\xn^{\top}\rvec(t)\right)<M$, in $(3)$ we used $\forall x:\ e^{x}\ge1+x$, $\vert\vect x_{k}^{\top}\rvec(t)\vert>C_{1}t^{-0.5\tilde{\mu}}$ and the fact that we can find $C'$ that satisfies the equation (since $\lim_{t\to\infty} \tilde{m}_n(t)=0$ and the other terms except $t^{-\mu_-}$ are constant).
    In $(4)$ we used that $C_{1}t^{-0.5\tilde{\mu}}$ decrease to zero slower
    than the other terms and therefore $\exists t_{-}>t_{1}^{-}$ such
    that $\forall t>t_{-}$ the last equation is greater than 1.
    \item If $\vert\vect x_{k}^{\top}\rvec(t)\vert>\epsilon_2$
    and $\vect x_{n}^{\top}\rvec(t)<0$ then $\exists t'_{-}\ge t_-, M''>1$ so that
    \begin{align*}
         &\left(1-\e\left(-\mu_{-}\xn^{\top}\wvec(t)\right)\right)\exp\left(\tilde{m}_n(t)\right)\exp\left(-\mathbf{x}_{n}^{\top}\mathbf{r}\left(t\right)\right)\\
         &\ge \left(1-\e\left(-\mu_{-}\xn^{\top}\wvec(t)\right)\right)\exp\left(\tilde{m}_n(t)\right)\exp\left(\epsilon_2\right)\ge M''>1
    \end{align*}
    and thus eq. \ref{eq: individual term sv}
    can be upper bounded $\forall t\ge t_{-}'$ by
    \begin{align} \label{eq: sv term negative improved bound}
     & -\min_n\alpha_{n}\frac{K}{t}\left(M''-1\right)\epsilon_2 \triangleq -C_{-}'' t^{-1}\,.
    \end{align}
    where we defined $C_{-}''=\min_n\alpha_{n}K\left(M''-1\right)\epsilon_2$.
\end{enumerate}
In conclusion, $\forall t\ge\max\left(t_{+}',t_{-}',t_{3}\right)$
\begin{enumerate}
    \item Each term in eq. \ref{eq: (r(t+1)-r(t))r(t)} can be upper bounded by either zero or a term proportional to $t^{-\theta}$ or $t^{-1-0.5\tilde{\mu}}$. Thus, we can find positive constants $C_{2},\ C_{3}$ such that
    \[
    \left(\mathbf{r}\left(t+1\right)-\mathbf{r}\left(t\right)\right)^{\top}\mathbf{r}\left(t\right)\le C_{2}t^{-\theta}+C_{3}t^{-1-0.5\tilde{\mu}}\,.
    \]
    \item If, in addition, $\norm{\op\rvec\left(t\right)}\ge\epsilon_1$ (eq. \ref{eq: Pr(t) lower bouded}), we have that
    \[
        \max_{n\in\set\cap\mathcal{B}\left(t\right)} \abs{\xnT\rvec\left(t\right)}^2 \overset{(1)}{\ge}
        \frac{1}{\abs{\set}} \sum_{n\in\set\cap\mathcal{B}\left(t\right)} \abs{\xnT\op\rvec\left(t\right)}^2 \overset{(2)}{=} \frac{1}{\abs{\set}}\norm{\vect{X}_{\set_1(t)}\op\rvec\left(t\right)}^2\overset{(3)}{\ge} \frac{1}{\abs{\set}}\epsilon_2'\epsilon_1^2\,,
    \]
    where in (1) we used $\forall n\in\set:\ \op^\top\xn=\xn$, in (2) we defined $\set_1(t)=\set\cap\mathcal{B}\left(t\right)$ and used this to define $\vect{X}_{\set_1(t)}$ as the matrix whose columns are a subset $\set_1\subset \left\{1,...,N\right\}$ of the columns of $\vect{X}=\left[\vect{x}_1,...,\vect{x}_N\right]\in\mathbb{R}^{d\times N}$. In (3) we used eq. \ref{eq: Pr(t) lower bouded} and also the fact that, for almost every data set, the support vectors are linearly independent and thus $\forall t:\ \lambda_{\min}\left( \vect{X}_{\set_1(t)}^\top\vect{X}_{\set_1(t)}\right)>0$. This implies that $\exists \epsilon_2'>0$ such that $\forall t:\ \lambda_{\min}\left( \vect{X}_{\set_1(t)}^\top\vect{X}_{\set_1(t)}\right)\ge\epsilon_2'>0$.
    Therefore, for some $n\in\set\cap\mathcal{B}\left(t\right)$, $\abs{\xnT\rvec\left(t\right)}>\epsilon_2=\sqrt{\abs{\set}^{-1}\epsilon_2'\epsilon_1^2}$.
    We define $C'' = \min\left(C_{+}'',C_{-}''\right)$. Using eqs. \ref{eq: non sv upper bound}, \ref{eq: sv term upper bound when xr is bounded}, \ref{eq: sv term positive improved bound} and \ref{eq: sv term negative improved bound}, we obtain,
    \[
        \left(\mathbf{r}\left(t+1\right)-\mathbf{r}\left(t\right)\right)^{\top}\mathbf{r}\left(t\right)\le -C''t^{-1}+o(t^{-1}).
    \]
    This implies that $\exists C_4, \bar{t}_2\ge\max\left(t_{+}',t_{-}',t_{3}\right)$ so that $\forall t\ge\bar{t}_2$ we have
    \[
        \left(\mathbf{r}\left(t+1\right)-\mathbf{r}\left(t\right)\right)^{\top}\mathbf{r}\left(t\right)\le -C_4t^{-1}.
    \]
\end{enumerate}
\QEDA

\subsection{Proof of Lemma \ref{lem:: sum sv aux lemma}\label{sec:Proof-of-Lemma}}
\sumSVAux*

\subsubsection{Proof for random sampling with replacement}

We define $z_{t,n}$ as the random variable equal to $1$ if sample
$n$ is selected at iteration $t$, and $0$ otherwise. Using this
variable, we can write
\begin{equation}
K\sum_{u=1}^{t-1}\frac{1}{u}\sum_{n\in\set\cap\mathcal{B}\left(u\right)}\alpha_{n}\mathbf{x}_{n}=K\sumnsv\sum_{u=1}^{t-1}\frac{z_{u,n}}{u}\alpha_{n}\mathbf{x}_{n}\,.\label{eq: Lemma 5 with replacment eq zero}
\end{equation}
We note that $\E z_{u,n}=K^{-1}$, and therefore,
\begin{align}
K\sum_{u=1}^{t-1}\frac{z_{u,n}}{u} & =K\sum_{u=1}^{t-1}\frac{\E z_{u,n}}{u}+K\sum_{u=1}^{t-1}\frac{z_{u,n}-\E z_{u,n}}{u}\nonumber \\
 & =\log\left(t\right)+\gamma+K\sum_{u=1}^{\infty}\frac{z_{u,n}-\E z_{u,n}}{u}-K\sum_{u=t}^{\infty}\frac{z_{u,n}-\E z_{u,n}}{u}+O\left(t^{-1}\right)\label{eq: Lemma 5 with replacment first eq}
\end{align}
where in the last we used the relations 
\[
\sum_{n=1}^{M}\frac{1}{m}=\log M+\gamma+O\left(\frac{1}{M}\right)\,,
\]
\[
\forall c:\,\log\left(\frac{t+c}{K}\right)-\log\left(\frac{t}{K}\right)=\log\left(1+ct^{-1}\right)=O\left(t^{-1}\right)\,,
\]
where $\gamma$ is the Euler-Mascheroni constant. Next, we bound the
remaining terms. We examine the value of the infinite sum for all $n$:
\[
\sum_{u=1}^{\infty}\frac{z_{u,n}-\E z_{u,n}}{u} \,.
\]

Since
\[
-\frac{1}{u}\le \frac{z_{u,n}-\E z_{u,n}}{u}\leq\frac{1}{u}\,,
\]
we have, by the Hoeffding inequality, that $\forall c$
\[
P\left(\left|\sum_{u=1}^{T}\frac{z_{u,n}-\E z_{u,n}}{u}\right|\ge c\right)\leq2\exp\left(-\frac{2c^{2}}{4\sum_{u=1}^{T}\frac{1}{u^2}}\right)\,.
\]Taking $T$ and $c$ to infinity and using $\sum_{k=1}^{\infty}\frac{1}{k^{2}}=\frac{\pi^{2}}{6}$ we get that 
this sum is convergent with probability $1$, i.e. 
\begin{equation}
P\left(\left|\sum_{u=1}^{\infty}\frac{z_{u,n}-\E z_{u,n}}{u}\right|<\infty\right)=1\,.\label{eq: Lemma 5 with replacment second eq}
\end{equation}

Next we examine the tail sum. Note that
\[
\sum_{k=t}^{\infty}\frac{1}{k^{2}}=\frac{1}{t}+O\left(\frac{1}{t^{2}}\right)
\]
and therefore, by the Hoeffding inequality,
\[
P\left(\left|\sum_{u=t}^{\infty}\frac{z_{u,n}-\E z_{u,n}}{u}\right|>c\right)\leq2\exp\left(-\frac{2tc^{2}}{1+O\left(t^{-1}\right)}\right)\,.
\]
or
\[
P\left(\left|\sum_{u=t}^{\infty}\frac{z_{u,n}-\E z_{u,n}}{u}\right|>\frac{c}{\sqrt{t}}\right)\leq2\exp\left(-\frac{2c^{2}}{1+O\left(t^{-1}\right)}\right)\,.
\]
Taking $c$ to infinity we obtain that, with probability $1$, $\forall\epsilon>0$
\begin{equation}
\sum_{u=t}^{\infty}\frac{z_{u,n}-\E z_{u,n}}{u}=o\left(t^{-0.5+\epsilon}\right)\,.\label{eq: Lemma 5 with replacment third eq}
\end{equation}
Recalling that $\what=\sum_{n\in\set}\alpha_{n}\xn$, and denoting
\[
\vect m_{1}(t)\triangleq K\sum_{u=t}^{\infty}\frac{z_{u,n}-\E z_{u,n}}{u}+O\left(t^{-1}\right)
\]
\[
\check{\mathbf{w}}\triangleq K\sum_{n\in\set\cap\mathcal{B}\left(u\right)}\sum_{u=1}^{\infty}\frac{z_{t,n}-\E z_{t,n}}{u}\alpha_{n}\mathbf{x}_{n}+\left(\log\left(K\right)+\gamma\right)\what,
\]
we combine this with eq. \ref{eq: Lemma 5 with replacment third eq},
\ref{eq: Lemma 5 with replacment second eq}, \ref{eq: Lemma 5 with replacment first eq}
into eq. \ref{eq: Lemma 5 with replacment eq zero}. This proves the
Lemma since $\check{\mathbf{w}}$ is a finite constant with probability
1, and $\vect m_{1}(t)=o\left(t^{-0.5+\epsilon}\right)$ and $\vect m_{1}(t+1)-\vect m_{1}(t)=O\left(t^{-1}\right)$.
\QEDB

\subsubsection{Proof for sampling without replacement}

\begin{align}
 & K\sum_{u=1}^{t-1}\frac{1}{u}\sum_{n\in\set\cap\mathcal{B}\left(u\right)}\alpha_{n}\mathbf{x}_{n}\nonumber \\
 & =K\sum_{u=1}^{K\left\lfloor \frac{t-1}{K}\right\rfloor }\frac{1}{u}\sum_{n\in\set\cap\mathcal{B}\left(u\right)}\alpha_{n}\mathbf{x}_{n}+K\sum_{u=K\left\lfloor \frac{t-1}{K}\right\rfloor +1}^{t-1}\frac{1}{u}\left(\sum_{n\in\set\cap\mathcal{B}\left(u\right)}\alpha_{n}\mathbf{x}_{n}\right)\nonumber \\
 & =\sum_{k=1}^{\left\lfloor \frac{t-1}{K}\right\rfloor }\sum_{n\in\set}\frac{1}{k-1+u_{n,k}/K}\alpha_{n}\mathbf{x}_{n}+\vect m_{1}\left(t\right) \, ,\label{eq: first eq in lemma 5}
\end{align}
where in the last line we recall we defined $u_{n,k}$ as the index
of the $n$'th example minibatch in the $k$'th epoch and therefore
$1\le u_{n,k}\le K$, and use the fact that we can write the second
term as $\vect m_{1}\left(t\right)$ since it is $O\left(t^{-1}\right)$
and is also is difference (i.e., $\vect m_{1}(t+1)-\vect m_{1}(t)=O\left(t^{-1}\right)$)
--- it is a finite sum of $Cu^{-1}$ terms where $u\geq t$-2. Next,
we examine the remaining term for a given $n$: 
\begin{align}
 & \sum_{k=1}^{\left\lfloor \frac{t-1}{K}\right\rfloor }\frac{1}{k-1+u_{n,k}/K}=\sum_{k=1}^{\left\lfloor \frac{t-1}{K}\right\rfloor }\left[\frac{1}{k}+\frac{K-u_{n,k}}{k^{2}K+(u_{n,k}-K)k}\right]\nonumber \\
 & =\log\left(\left\lfloor \frac{t-1}{K}\right\rfloor \right)+\gamma+O\left(t^{-1}\right)+\sum_{k=1}^{\left\lfloor \frac{t-1}{K}\right\rfloor }\frac{K-u_{n,k}}{k^{2}K+(u_{n,k}-K)k}\label{eq: second eq in lemma 5}
\end{align}
where in the last line we used the relations 
\[
\sum_{n=1}^{M}\frac{1}{m}=\log M+\gamma+O\left(\frac{1}{M}\right)\,,
\]
\[
\forall c:\,\log\left(\frac{t+c}{K}\right)-\log\left(\frac{t}{K}\right)=\log\left(\frac{1+ct^{-1}}{K}\right)=O\left(t^{-1}\right)\,,
\]
where $\gamma$ is the Euler-Mascheroni constant. We examine the remaining
sum in eq. \ref{eq: second eq in lemma 5}. Since $1\le u_{n,k}\le K,$
each term in the sum is $\Theta\left(k^{-2}\right)$, which implies
that this sum is convergent, and we can write it as
\[
\sum_{k=1}^{\left\lfloor \frac{t-1}{K}\right\rfloor }\frac{K-u_{n,k}}{k^{2}K+(u_{n,k}-K)k}=\sum_{k=1}^{\infty}\frac{K-u_{n,k}}{k^{2}K+(u_{n,k}-K)k}+O\left(t^{-1}\right)\,.
\]
Recalling that $\what=\sum_{n\in\set}\alpha_{n}\xn$, defining 
\[
\check{\mathbf{w}}\triangleq\sum_{n\in\set}\sum_{k=1}^{\infty}\frac{K-u_{n,k}}{k^{2}K+(u_{n,k}-K)k}\alpha_{n}\mathbf{x}_{n}+\gamma\what
\]
and combining this into eq. \ref{eq: first eq in lemma 5}, using
eq. \ref{eq: second eq in lemma 5}, we prove the Lemma. \QEDB
\remove{
\newpage
\section{Additional Empiric Results} \label{sec: more emipirical results}
In figure \ref{fig: DNN results} we observe the learning dynamics of a ResNet-18 trained on CIFAR10 using SGD with momentum. Particularly, we observe that even though the learning rate is fixed, the training loss and classification error converge to zero.
In the next figure we demonstrate similar results when using SGD without momentum.
\begin{figure*}[h]
\begin{centering}
\begin{tabular}{cc}
\includegraphics[width=0.48\columnwidth]{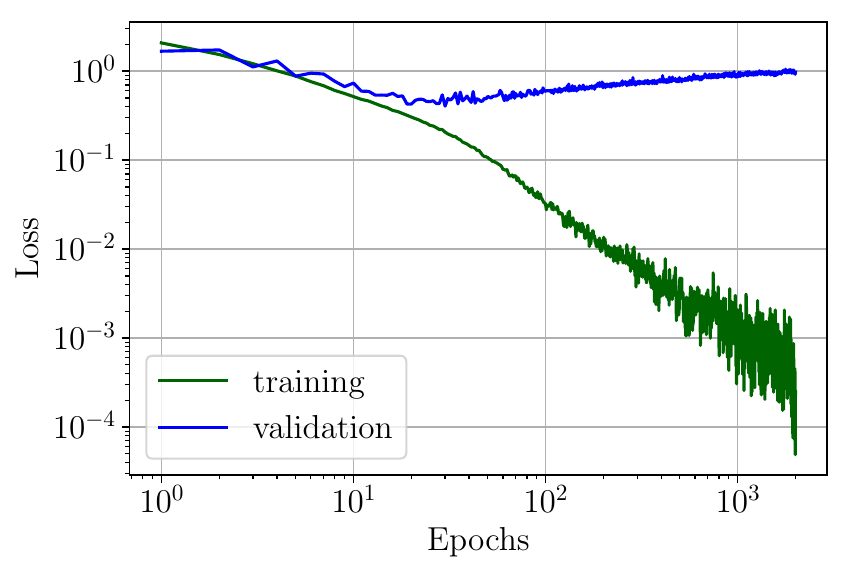}  & \includegraphics[width=0.48\columnwidth]{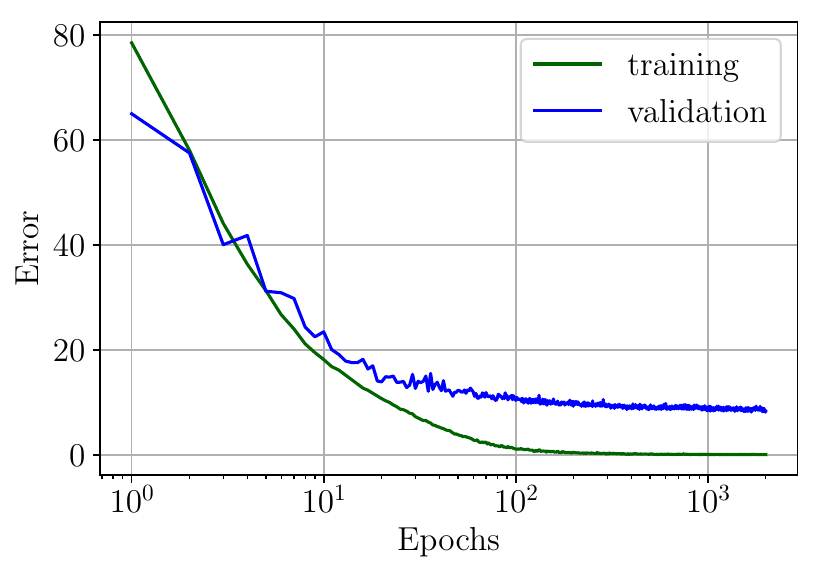}  \tabularnewline
\end{tabular}
\par\end{centering}
\caption{Training of a convolutional neural network on CIFAR10 using stochastic
gradient descent with constant learning rate and \underline{without} momentum, softmax
output and a cross entropy loss, where we achieve $8.3\%$ final validation
error. We observe that, approximately: (1) The training loss and (classification) error both decays to zero; (2) after a while, the validation loss starts to increase; and (3) in contrast, the validation (classification) error slowly improves.}
\end{figure*}
}
\newpage
\section{Additional empirical results}
\remove{
\begin{figure*}[h]
\begin{centering}
\includegraphics[width=0.95\columnwidth]{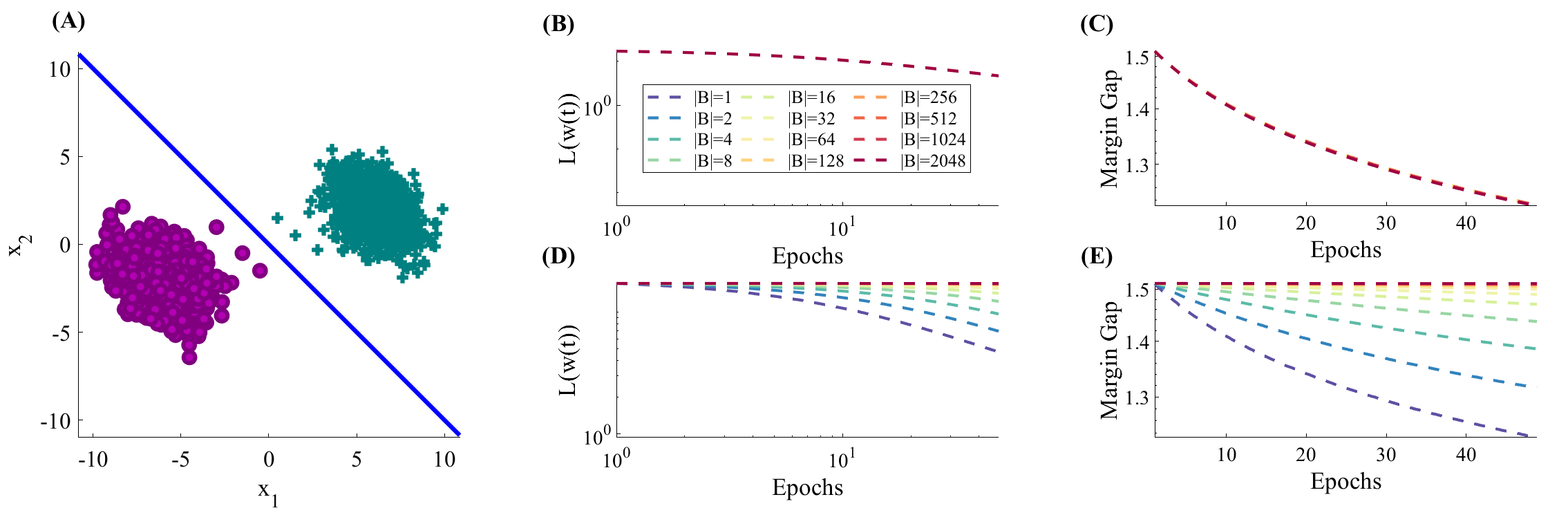}
\includegraphics[width=0.95\columnwidth]{Plots/sgd_differentBatchSizes_n2048_offsetDataset2_logScale_1e6Iter.png}
\end{centering}
\caption{We repeat the experiment observed in Figure \ref{fig: minibatch} with different datasets. In both figures $\eta=\frac{2\gamma^{2}}{\beta\sigma_{\max}^{2}}B$ in panels \textbf{B} and \textbf{C}, vs. $\eta=\frac{2\gamma^{2}}{\beta\sigma_{\max}^{2}}$ in panels \textbf{D} and \textbf{E}. We initialized $\mathbf{w}(0)$ to be a standard normal vector. We used a random dataset  \textbf{(A)} with $N=2048$ samples divided into two classes, and with the same support vectors as in Figure \ref{fig:Synthetic-dataset}. We can see that in the top figure the convergence of the loss \textbf{(B)} and margin \textbf{(C)} is practically identical for all minibatch sizes in the top figure and is different initially in the bottom figure but is asymptotically identical for all batch sizes. When we used a fixed learning rate, the convergence rate was different \textbf{(D-E)}. In contrast, we can see that when generating the samples with different seed the convergence of the loss \textbf{(B)} and margin \textbf{(C)} are not identical with larger minibatch sizes.} \label{fig: different batch sizes}
\end{figure*}
}
\begin{figure*}[h]
\begin{centering}
\includegraphics[width=0.95\columnwidth]{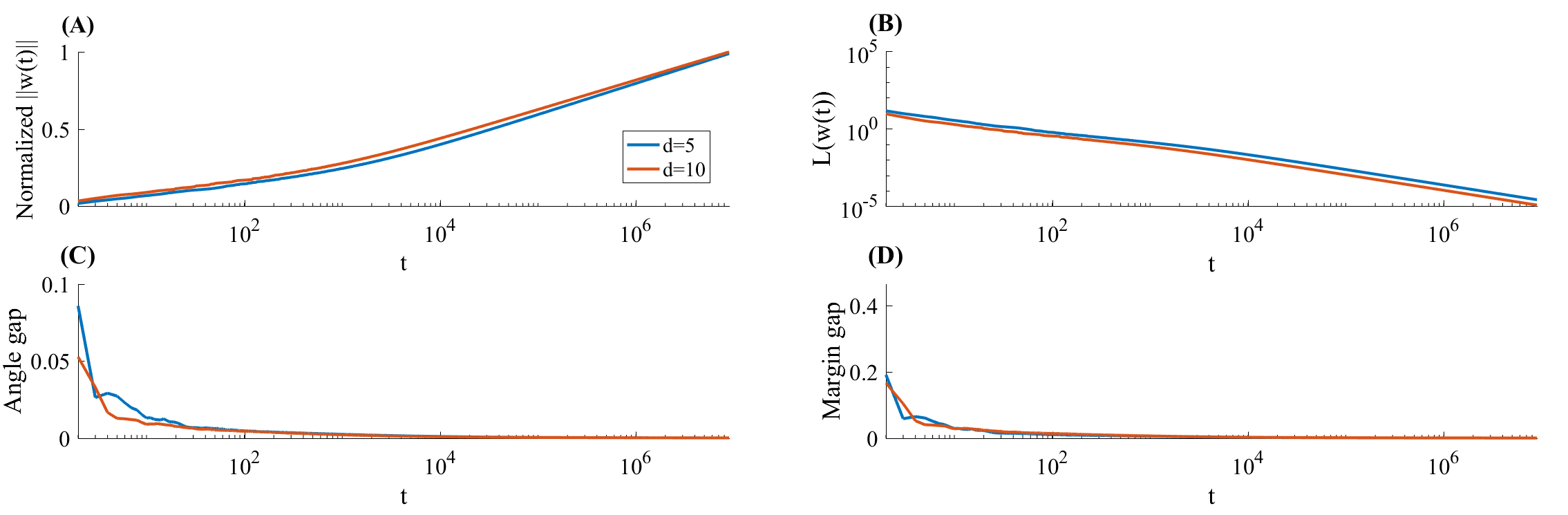}
\end{centering}
\caption{\textbf{Visualization of Theorem \ref{thm: direction convergence to SVM} on a synthetic datasets with dimension $d=5$ and $d=10$ in which the
$L_{2}$ max margin vector $\hat{\mathbf{w}}$ are precisely known}.
We show: \textbf{(B)
}
The norm of $\mathbf{w}\left(t\right)$, normalized so it would equal
to $1$ at the last iteration. As expected
(from eq. \ref{define wVec}), the norm increases logarithmically;
\textbf{(C) }the training loss. As expected, it decreases as $t^{-1}$
(eq. \ref{eq: logistic loss convergence}); and \textbf{(D\&E) }the
angle and margin gap of $\mathbf{w}\left(t\right)$ from $\hat{\mathbf{w}}$
(eqs. \ref{eq: angle} and \ref{eq: margin}). As expected, these
are logarithmically decreasing to zero. Figure reproduced from \citet{soudry2017implicit}. } \label{fig: SGD results in higher dimension}
\end{figure*}

\remove{
\begin{proof}
\begin{align*}
& K\sum_{u=1}^{t-1}\frac{1}{u}\sum_{n\in\set\cap\mathcal{B}\left(u\right)}\alpha_{n}\mathbf{x}_{n} \\ & = K\sum_{u=1}^{K\left\lfloor \frac{t-1}{K}\right\rfloor}\frac{1}{u}\sum_{n\in\set\cap\mathcal{B}\left(u\right)}\alpha_{n}\mathbf{x}_{n}
+K\sum_{u=K\left\lfloor \frac{t-1}{K}\right\rfloor +1}^{t-1}\frac{1}{u}\sum_{n\in\set\cap\mathcal{B}\left(u\right)}\alpha_{n}\mathbf{x}_{n}\\
&\overset{(1)}{=} K\sum_{k=1}^{\left\lfloor \frac{t-1}{K}\right\rfloor }\sum_{n\in\set}\frac{1}{(k-1)K+u_{n,k}}\alpha_{n}\mathbf{x}_{n}+K\sum_{u=K\left\lfloor \frac{t-1}{K}\right\rfloor +1}^{t-1}\frac{1}{u}\sum_{n\in\set\cap\mathcal{B}\left(u\right)}\alpha_{n}\mathbf{x}_{n}\\
 & \overset{(2)}{=}K\sum_{k=1}^{\left\lfloor \frac{t-1}{K}\right\rfloor }\sum_{n\in\set}\frac{1}{(k-1)K+u_{n,k}}\alpha_{n}\mathbf{x}_{n}+\vect m_{1}(t)\\
 & \overset{(3)}{=}K\sum_{k=1}^{\left\lfloor \frac{t-1}{K}\right\rfloor }\sum_{n\in\set}\left(\frac{1}{kK}+\frac{K-u_{n,k}}{k^{2}K^{2}+(u_{n,k}-K)kK}\right)\alpha_{n}\mathbf{x}_{n}+\vect m_{1}(t)\\
 & \overset{(4)}{=}\sum_{k=1}^{\left\lfloor \frac{t-1}{K}\right\rfloor }\frac{1}{k}\sum_{n\in\set}\alpha_{n}\mathbf{x}_{n}+\vect m_{1}(t)+\vect m_{2}(t)\\
 & \overset{(5)}{=}\log\left(\left\lfloor \frac{t-1}{K}\right\rfloor \right)\hat{\mathbf{w}}+\gamma\what+\vect m_{3}(t)\\
 & =\log\left(\frac{t}{K}\right)\hat{\mathbf{w}}+\gamma\what+\vect m_{3}(t)-\left(\log\left(\frac{t}{K}\right)-\log\left(\left\lfloor \frac{t-1}{K}\right\rfloor \right)\right)\what\\
 & \overset{(6)}{=}\log\left(\frac{t}{K} \right)\hat{\mathbf{w}}+\gamma\what+\vect m_{4}(t)
\end{align*}
where in (1) we defined $u_{n,k}$ as the index of the $n$'th example
minibatch in the $k$ epoch and therefore $1\le u_{n,k}\le K$.\\
In order to explain the second transition we note that $\forall b=1,...,K$:
\begin{equation*}
	\lim_{t\to\infty}\frac{t}{K\left\lfloor \frac{t-1}{K}\right\rfloor +b} = \lim_{t\to\infty}\frac{1}{K\nicefrac{\left\lfloor \frac{t-1}{K}\right\rfloor}{\frac{t-1}{K}} +\nicefrac{b}{\frac{t-1}{K}}} \frac{t}{\frac{t-1}{K}}=1
\end{equation*}
 which means that $\frac{1}{K\left\lfloor \frac{t-1}{K}\right\rfloor +b}=\Theta\left(\frac{1}{t}\right).$
This explains the second transition since in (2) we defined $\norm{\vect m_{1}(t)}=O\left(\frac{1}{t}\right)$
and used the fact that $\sum_{u=K\left\lfloor \frac{t-1}{K}\right\rfloor +1}^{t-1}\frac{1}{u}\sum_{n\in\set\cap\mathcal{B}\left(u\right)}\alpha_{n}\mathbf{x}_{n}$
contains a finite number of terms (N at most) and each of theirs elements
is $O\left(\frac{1}{t}\right)$.\\
In 3 we used the relation:
\begin{align*}
 & \frac{1}{1+x}=1-\frac{x}{1+x}\\
 & \Rightarrow\frac{1}{(k-1)K+u_{n,k}}=\frac{1}{kK}\cdot\frac{1}{1+\frac{u_{n,k}-K}{kK}}\\ &=\frac{1}{kK}\cdot\left(1-\frac{u_{n,k}-K}{kK+u_{n,k}-K}\right)=\frac{1}{kK}+\frac{K-u_{n,k}}{k^{2}K^{2}+(u_{n,k}-K)kK}.
\end{align*}
In 4 we defined $\norm{\vect m_{2}(t)}=O\left(\frac{1}{t}\right)$
and used the fact that $\exists k_{1},C_{1}$ positive constant so
that $\forall k\ge k_{1}$

\[
0\le\frac{K-u_{n,k}}{k^{2}K^{2}+(u_{n,k}-K)kK}\le\frac{K-1}{k^{2}K^{2}+(1-K)kK}\le\frac{C_{1}}{k^{2}}
\]
since

\begin{align*}
&\lim_{k\to\infty}\frac{(K-1)k^{2}}{k^{2}K^{2}+(1-K)kK}=\lim_{k\to\infty}\frac{K-1}{K^{2}+(1-K)K/k}=\frac{K-1}{K^{2}}\\
&\Rightarrow\frac{K-1}{k^{2}K^{2}+(1-K)kK}=\Theta\left(\frac{1}{k^{2}}\right).
\end{align*}
and $\sum_{k=1}^{\infty}\frac{1}{k^{2}}=O\left(\frac{1}{k}\right)$.\\
In 5 we defined $\norm{\vect m_{3}(t)}=O\left(\frac{1}{t}\right)$
and used the relation $\sum_{n=1}^{K}\frac{1}{n}=\log K+\gamma+O\left(\frac{1}{2K}\right)$,
where $\gamma$ is the Euler\textendash Mascheroni constant, and the
fact that $\what=\sum_{n\in\set}\alpha_{n}\xn$.
In 6 we used the fact that $\forall \epsilon \in(0,1)$
\begin{equation} \label{eq: asym floor(t-1) and t}
0\le t^{-(-1+\epsilon)}\log\left(\dfrac{\frac{t}{K}}{\left\lfloor \frac{t-1}{K}\right\rfloor }\right)
\le t^{-(-1+\epsilon)}\log\left(\dfrac{\frac{t}{K}}{ \frac{t-1}{K}-1 }\right) = t^{-(-1+\epsilon)}\log\left(1+\frac{K+1}{t-1-K}\right) \xrightarrow{t\to\infty} 0
\end{equation}
since
\begin{align*}
	&\lim_{t\to\infty} \frac{\log\left(1+\frac{K+1}{t-1-K}\right)}{t^{-1+\epsilon}} 
	\overset{(1)}{=} \lim_{t\to\infty} \frac{-(K+1)(t-1-K)^{-2}}{1+\frac{K+1}{t-1-K}}\cdot \frac{1}{(\epsilon-1)t^{-2+\epsilon}}\\
	&= \lim_{t\to\infty} \frac{(K+1)t^2}{(t-1-K)^{2}+(K+1)(t-1-K)}\cdot \frac{1}{(1-\epsilon)t^{\epsilon}}=0,
	\end{align*}
	where in 1 we used L'Hospital's rule.
	From eq. \ref{eq: asym floor(t-1) and t} and Squeeze Theorem:
	\begin{equation*}
		\forall \epsilon\in(0,1):\ \lim_{t\to\infty} t^{-(-1+\epsilon)}\log\left(\dfrac{\frac{t}{K}}{\left\lfloor \frac{t-1}{K}\right\rfloor }\right) = 0
	\end{equation*}  
	and therefore $\forall\epsilon\in(0,1)\ :\ \log\left(\dfrac{\frac{t}{K}}{\left\lfloor \frac{t-1}{K}\right\rfloor }\right) = o(t^{-1+\epsilon})$.
\end{proof}
}

\begin{thebibliography}{23}
\expandafter\ifx\csname natexlab\endcsname\relax\def\natexlab#1{#1}\fi
\expandafter\ifx\csname url\endcsname\relax
  \def\url#1{{\tt #1}}\fi

\bibitem[Bach and Moulines(2011)]{Bach2011a}
Francis Bach and Eric Moulines.
\newblock {Non-Asymptotic Analysis of Stochastic Approximation Algorithms for
  Machine Learning}.
\newblock {\em NIPS}, pages~--, 2011.

\bibitem[Bassily et~al.(2018)Bassily, Belkin, and Ma]{Bassily2018}
Raef Bassily, Mikhail Belkin, and Siyuan Ma.
\newblock {On exponential convergence of SGD in non-convex over-parametrized
  learning}.
\newblock pages 1--7, 2018.

\bibitem[Ben-David and Shalev-Shwartz(2014)]{Ben-David2014}
Shai Ben-David and Shai Shalev-Shwartz.
\newblock {\em {Understanding Machine Learning: From Theory to Algorithms}}.
\newblock 2014.

\bibitem[Bertsekas(1999)]{Bertsekas1999}
D.~Bertsekas.
\newblock {\em {Nonlinear Programming}}.
\newblock Athena Scientific, 1999.

\bibitem[Bertsekas(2011)]{Bertsekas2015}
Dimitri~P. Bertsekas.
\newblock {Incremental proximal methods for large scale convex optimization}.
\newblock {\em Mathematical Programming}, 129\penalty0 (2):\penalty0 163--195,
  jul 2011.

\bibitem[Bottou et~al.(2016)Bottou, Curtis, and Nocedal]{Bottou2016}
L{\'{e}}on Bottou, Frank~E. Curtis, and Jorge Nocedal.
\newblock {Optimization Methods for Large-Scale Machine Learning}.
\newblock 2016.

\bibitem[Bubeck(2015)]{Bubeck2015}
S{\'{e}}bastien Bubeck.
\newblock {Convex optimization: Algorithms and complexity}.
\newblock {\em Foundations and Trends{\textregistered}in Machine Learning},
  8\penalty0 (3-4):\penalty0 231--357, 2015.

\bibitem[Geary and Bertsekas(2001)]{Geary2001}
A.~Geary and D.P. Bertsekas.
\newblock {Incremental subgradient methods for nondifferentiable optimization}.
\newblock {\em Proceedings of the 38th IEEE Conference on Decision and Control
  (Cat. No.99CH36304)}, 1\penalty0 (1):\penalty0 907--912, 2001.

\bibitem[Ghadimi et~al.(2013)Ghadimi, Lan, and Zhang]{Ghadimi2013}
Saeed Ghadimi, Guanghui Lan, and Hongchao Zhang.
\newblock {Mini-batch Stochastic Approximation Methods for Nonconvex Stochastic
  Composite Optimization}.
\newblock {\em Math. Prog.}, 155\penalty0 (1-2):\penalty0 267--305, 2013.

\bibitem[Goyal et~al.(2017)Goyal, Doll{\'{a}}r, Girshick, Noordhuis,
  Wesolowski, Kyrola, Tulloch, Kaiming, and Facebook]{Goyal2017}
Priya Goyal, Piotr Doll{\'{a}}r, Ross Girshick, Pieter Noordhuis, Lukasz
  Wesolowski, Aapo Kyrola, Andrew Tulloch, Yangqing~Jia Kaiming, and
  He~Facebook.
\newblock {Accurate, Large Minibatch SGD: Training ImageNet in 1 Hour}.
\newblock {\em arXiv preprint}, 2017.

\bibitem[Gunasekar et~al.(2018{\natexlab{a}})Gunasekar, Lee, Soudry, and
  Srebro]{Gunasekar2018}
Suriya Gunasekar, Jason Lee, Daniel Soudry, and Nathan Srebro.
\newblock {Implicit Bias of Gradient Descent on Linear Convolutional Networks}.
\newblock In {\em NIPS}, jun 2018{\natexlab{a}}.

\bibitem[Gunasekar et~al.(2018{\natexlab{b}})Gunasekar, Lee, Soudry, and
  Srebro]{gunasekar2018implicit}
Suriya Gunasekar, Jason~D. Lee, Daniel Soudry, and Nathan Srebro.
\newblock Characterizing implicit bias in terms of optimization geometry.
\newblock In {\em ICML}, 2018{\natexlab{b}}.

\bibitem[Hoffer et~al.(2017)Hoffer, Hubara, and Soudry]{hoffer2017train}
Elad Hoffer, Itay Hubara, and Daniel Soudry.
\newblock Train longer, generalize better: closing the generalization gap in
  large batch training of neural networks.
\newblock In {\em Advances in Neural Information Processing Systems}, pages
  1729--1739, 2017.

\bibitem[Jastrzebski et~al.(2017)Jastrzebski, Kenton, Arpit, Ballas, Fischer,
  Bengio, and Storkey]{Jastrzebski2017}
Stanislaw Jastrzebski, Zachary Kenton, Devansh Arpit, Nicolas Ballas, Asja
  Fischer, Yoshua Bengio, and Amos Storkey.
\newblock {Three Factors Influencing Minima in SGD}.
\newblock {\em arXiv}, pages 1--21, 2017.

\bibitem[Ji and Telgarsky(2018)]{ji2018risk}
Ziwei Ji and Matus Telgarsky.
\newblock Risk and parameter convergence of logistic regression.
\newblock {\em arXiv preprint arXiv:1803.07300v2}, 2018.

\bibitem[Ma et~al.(2017)Ma, Bassily, and Belkin]{Ma2017}
Siyuan Ma, Raef Bassily, and Mikhail Belkin.
\newblock {The Power of Interpolation: Understanding the Effectiveness of SGD
  in Modern Over-parametrized Learning}.
\newblock 2017.

\bibitem[Nacson et~al.(2019)Nacson, Lee, Gunasekar, Srebro, and
  Soudry]{Nacson2018}
Mor~Shpigel Nacson, Jason Lee, Suriya Gunasekar, Nathan Srebro, and Daniel
  Soudry.
\newblock {Convergence of Gradient Descent on Separable Data}.
\newblock {\em AISTATS}, 2019.

\bibitem[Robbins and Monro(1951)]{Robbins1951}
Herbert Robbins and Sutton Monro.
\newblock {A Stochastic Approximation Method}.
\newblock {\em The Annals of Mathematical Statistics}, 22\penalty0
  (3):\penalty0 400--407, 1951.

\bibitem[Shamir(2016)]{Shamir2016}
Ohad Shamir.
\newblock {Without-Replacement Sampling for Stochastic Gradient Methods:
  Convergence Results and Application to Distributed Optimization}.
\newblock pages 1--36, 2016.

\bibitem[Smith et~al.(2018)Smith, Kindermans, Ying, and Le]{SmithLe2018}
Samuel~L. Smith, Pieter-Jan Kindermans, Chris Ying, and Quoc~V. Le.
\newblock {Don't Decay the Learning Rate, Increase the Batch Size}.
\newblock In {\em ICLR}, 2018.

\bibitem[Soudry et~al.(2018{\natexlab{a}})Soudry, Hoffer, {Shpigel Nacson},
  Gunasekar, and Srebro]{soudry2018journal}
Daniel Soudry, Elad Hoffer, Mor {Shpigel Nacson}, Suriya Gunasekar, and Nathan
  Srebro.
\newblock The implicit bias of gradient descent on separable data.
\newblock {\em arXiv preprint: 1710.10345v3}, 2018{\natexlab{a}}.

\bibitem[Soudry et~al.(2018{\natexlab{b}})Soudry, Hoffer, and
  Srebro]{soudry2017implicit}
Daniel Soudry, Elad Hoffer, and Nathan Srebro.
\newblock The implicit bias of gradient descent on separable data.
\newblock {\em ICLR}, 2018{\natexlab{b}}.

\bibitem[Xu et~al.(2018)Xu, Zhou, Ji, and Liang]{Xu2018}
Tengyu Xu, Yi~Zhou, Kaiyi Ji, and Yingbin Liang.
\newblock {When Will Gradient Methods Converge to Max-margin Classifier under
  ReLU Models?}
\newblock {\em arXiv}, 2018.

\end{thebibliography}
\end{document}